%% file: root.tex
\documentclass[letterpaper, 10 pt, conference]{ieeeconf}  

\IEEEoverridecommandlockouts                              

\usepackage{graphics} 
\usepackage{graphicx}
\usepackage{color}
\usepackage{amsmath} 
\usepackage{amssymb}  
\usepackage{amsthm}
\usepackage{mathtools}
\usepackage{verbatim}
\usepackage{soul}
\usepackage{mathabx} 

\usepackage{caption}
\usepackage{subcaption}

\usepackage{algorithm} 
\usepackage{algorithmic}

\usepackage{upgreek}
\usepackage{bm}
\renewcommand{\b}[1]{\boldsymbol{#1}}

\newcommand\h[1]{#1}
\newcommand\hu[1]{\textcolor{black}{#1}}

\newcommand{\figuretag}[1]{%
  \addtocounter{figure}{-1}%
  \renewcommand{\thefigure}{#1}%
}

\newcommand{\DG}[1]{{\textcolor{black}{{#1}}}}

\newtheorem{theorem}{Theorem}

\newtheorem{assumption}{Assumption}

\newtheoremstyle{remark_custom}
    {0} 
    {0} 
    {} 
    {} 
    {\it} 
    {.} 
    {.5em} 
    {} 

\theoremstyle{remark_custom}
\newtheorem{remark}{Remark}

\begin{document}

\title{\LARGE \bf
Scenario-Based \h{Trajectory Optimization} in\\ Uncertain Dynamic Environments
}

\author{Oscar de Groot$^{*}$, Bruno Brito$^{*}$, Laura Ferranti$^{*}$, Dariu Gavrila$^{*}$ and Javier Alonso-Mora$^{*}$
\thanks{$^*$The authors are with the \DG{Dept.} of Cognitive Robotics, TU Delft, 2628 CD Delft, The Netherlands. {\DG{\texttt {Email: o.m.degroot@tudelft.nl}}}}
\thanks{{This work received support from the Dutch Science Foundation \DG{NWO-TTW}, within the SafeVRU project (nr. 14667) and Veni award (nr. 15916).}}
}

\maketitle

\input{content/abstract}
\input{content/introduction}

\input{content/problem_formulation}
\input{content/preliminary}
\input{content/method}
\input{content/implementation}

\input{content/experiments}
\input{content/conclusion}

\bibliographystyle{IEEEtran}
\bibliography{references.bib}

\end{document}

%% file: content/abstract.tex
\begin{abstract}
We present an optimization-based method to plan the motion of an autonomous robot under the uncertainties associated with dynamic obstacles, such as humans. Our method bounds the marginal risk of collisions at each point in time by incorporating chance constraints into the planning problem. This problem is not suitable for online optimization outright for arbitrary probability distributions. Hence, we sample from these chance constraints \hu{using an uncertainty model,} to generate "scenarios", which translate the probabilistic constraints into deterministic ones. In practice, each scenario represents the collision constraint for a dynamic obstacle at the location of the sample. The number of theoretically required scenarios can be very large. Nevertheless, by exploiting the geometry of the workspace, we show how to prune most scenarios before optimization and we demonstrate how the reduced scenarios can still provide probabilistic guarantees on the safety of the motion plan. Since our approach is scenario based, we are able to handle arbitrary uncertainty distributions. We apply our method in a Model Predictive Contouring Control framework and demonstrate its benefits in simulations and experiments with a moving robot platform navigating among pedestrians, running in real-time.
\end{abstract}

%% file: content/introduction.tex
\section{INTRODUCTION} 
Mobile robots are 
\DG{increasingly} becoming part of our society, with applications in warehouses~\cite{simon_inside_2019}, automotive~\cite{walker_self-driving_2019}, maritime transportation~\cite{mi_news_network_7_2018}, etc. In all these \DG{domains}, it is \DG{essential} that the robots can safely operate in dynamic environments (e.g., near humans). However, uncertainty is omnipresent\h{, for example, in the future motion paths of the dynamic obstacles or in sensing (i.e., localization) errors.}
%
Our goal is to design a local robot motion planning \DG{algorithm} able to plan collision-free trajectories in the presence of possibly unbounded and arbitrary uncertainties.\looseness=-1 

Optimization-based motion planning methods avoid collisions by imposing constraints in the optimization problem. Classical methods consider deterministic obstacle predictions, that is, they do not account for the presence of uncertainties. When uncertainties come into the picture, deterministic frameworks fail to achieve safety, since they do not consider the possible spread of outcomes. In the case of bounded uncertainties, that is, if the probability density function is non-zero in a bounded domain of the robot's workspace and is zero elsewhere, then it is possible to set the acceptable level of risk to zero. This approach is referred to as \emph{robust optimization}. On the one hand, this approach allows for the addition of uncertainties in the deterministic framework. On the other hand, the assumption that the distribution is bounded can be limiting (\hu{e.g., when} obstacle predictions are Gaussian). Additionally, it becomes conservative when the domain of support is large. In the presence of unbounded uncertainties, \textit{chance constraint optimization} allows one to constrain the probability of collisions to be below an acceptable level of risk. \h{In this work, and likewise to \cite{zhu_chance-constrained_2019, schildbach_randomized_2013}, we consider the marginal probabilities of collision at each point in time. This is, we constrain the chance of collision for each step of the trajectory, separately.}

Directly evaluating these chance constraints is intractable, especially for arbitrary shapes of the distribution. Instead they are often either approximated (e.g., using particle filters \cite{blackmore_probabilistic_2010}) or bounded. \hu{Approximation techniques have} received most attention, due to their sample efficiency. \hu{However, the safety of these approaches cannot be guaranteed, especially when operating in unknown environments.}



\paragraph*{Contribution} In this work, we assume that a perception module provides predictions of the motion of dynamic obstacles together with a description of their (unbounded, possibly non Gaussian) uncertainty. \h{To provide probabilistic safety for each step of the planned trajectory with respect to the \hu{modeled} uncertainties,}
our work presents a novel \emph{probabilistic trajectory optimization} framework for motion planning in uncertain dynamic environments, that is, a Scenario-based Model Predictive Contouring Control (S-MPCC) design. Our S-MPCC builds on nonconvex scenario-optimization framework~\cite{campi_general_2018} and the model predictive contouring control (MPCC) design of~\cite{brito_model_2019}. \h{We show that in contrast with the general a posteriori results in}~\h{\cite{campi_general_2018}, we obtain the perceived risk of our motion plan \textit{before optimization}. The support subsample, which is the key indicator for the risk in \cite{campi_general_2018}, is obtained through the geometry of the problem, leading to efficient evaluation of the samples. While sampling-based chance constrained approaches are generally considered intractable for real-time motion planning, our method is competitive in terms of computation times with state-of-the-art planning methods, while applicable to generic uncertainties. The approach handles multiple obstacles and accounts for the size of the vehicle and obstacles.}

We show how our approach allows the robot to move continuously through its environment while reasoning about its probability of colliding with dynamic obstacles.
In our framework\h{, illustrated in Fig. \ref{fig:removal_gaussian}}, instead of directly solving the chance constrained motion planning problem, we solve an associated deterministic problem obtained as follows. \h{First, we apply a tailored linearization of the chance constraints, then we sample from the linearized chance constraints a large set of deterministic constraints, known as \emph{scenarios}. 
The number of scenarios drawn is linked with the associated risk of collisions. This allows us to reformulate the original planning problem in a deterministic one, known as a scenario program. Using this approach we 
effectively resolve the chance constraints in a preprocessing step.}

\h{Uncertainty of predictions \hu{is} generally non Gaussian and appears, for example, when Gaussian uncertainty is propagated through nonlinear dynamics. Our method is applicable to generic uncertainties. We demonstrate our framework for Gaussian and non Gaussian uncertainties using an autonomous ground robot, both in simulation and in experiments.}

\paragraph*{Related Work} 
Trajectory optimization with chance constraints for collision avoidance has previously been considered in the case of \emph{Gaussian uncertainties}. For example,~\cite{brito_model_2019} defined collision avoidance constraints in an MPC framework, by modeling the dynamic obstacles as ellipses. This representation allows the planner to accommodate Gaussian uncertainties (as the level set of a Gaussian distribution are ellipses) and solve a deterministic nonlinear optimization problem online. The approach was applied for autonomous driving in \cite{ferranti_safevru_2019}. In~\cite{zhu_chance-constrained_2019} the chance constraint problem is solved explicitly. Their approach linearizes the collision chance constraints, and uses the prior of Gaussian uncertainty to formulate deterministic constraints on the mean and covariance of the distribution.

Literature on motion planning for \textit{non Gaussian} uncertainty distributions is still limited. Inspired by particle filter approaches,~\cite{blackmore_probabilistic_2010} introduces a method to approximately evaluate the chance constraints using particles. Their method uses a relatively small sample size, but cannot provide any guarantees on the safety of the solution. In~\cite{berntorp_motion_2019}, a Rapidly-expanding Random Trees (RRT) algorithm is presented where the uncertainty is evaluated for each node in the tree. The dynamic obstacle trajectories are predicted using Gaussian Processes. 
In \cite{wang_non-gaussian_2020}, assuming full knowledge of the probability distribution, polynomial chance constraints are transformed to deterministic inequalities using the statistical moments of the non Gaussian distribution. 

Compared to the previous methods, our approach bounds the probability of collision using the scenario optimization framework. This framework is well established for convex optimization (\cite{calafiore_scenario_2006}, \cite{campi_exact_2008}, \cite{calafiore_random_2010}, \cite{campi_sampling-and-discarding_2011}, \cite{schildbach_randomized_2013}, \cite{schildbach_scenario_2014}). A framework for non-convex scenario optimization was recently introduced in \cite{campi_general_2018}. We rely on this framework and extend \h{it} in the context of robot motion planning.

%% file: content/problem_formulation.tex
\section{PROBLEM FORMULATION}\label{sec:problem_formulation}
We consider the motion planning problem of a mobile robot, whose dynamics can be represented by the following nonlinear discrete-time system:
\begin{equation}
    \b{x}_{k + 1} = f(\b{x}_k, \b{u}_k),
\end{equation}
where $\b{x}_k \in \mathbb{R}^n$ and $\b{u}_k \in \mathbb{R}^m$ denote the states and inputs, respectively. The robot can move within a workspace (e.g., the 2D plane when we consider ground robots).  In the workspace, the robot must avoid collisions with dynamic obstacles. 
We model the collision region of the robot $\mathcal{V}_k$ at time $k$ as the union of $n_c$ circles, and the collision region of the dynamic obstacles $\mathcal{D}^v_k$ at time $k$ as a single circle.  

The position of dynamic obstacles along the planning horizon of the robot is uncertain. \hu{We denote the uncertainty of the obstacles at stage $k$ with a tuple $(\Delta_k, \mathcal{D}_k, \mathbb{P}_{k,\textrm{real}})$, where $\Delta_k$ is a probability space equipped with a $\sigma$-algebra $\mathcal{D}_k$ and a probability measure $\mathbb{P}_{k,\textrm{real}}$.} We allow the probability spaces to be unbounded and non Gaussian. \hu{We assume that at each step a perception module provides the motion planner with an \hu{independent} model of the uncertainty, formalized as follows.}
\begin{assumption}\label{as:model}
\hu{The planner is provided with a model $\mathbb{P}_k$ of the real probability measure $\mathbb{P}_{k,\textrm{real}}$ for each $k$.}
\end{assumption}
\begin{assumption}\label{as:independent}
\hu{\h{Random variables $\bm{\delta}_j \hu{\sim \mathbb{P}_j}$ 
and $\bm{\delta}_l \hu{\sim \mathbb{P}_l}$}}
are independent for all \hu{stages} $j, l \in \{1, \hdots, N\}$, \h{where $j\neq l$.}
\end{assumption}
\h{Assumption \ref{as:independent} implies that the dependency induced, for example, by the dynamics of an obstacle, is handled by the perception module such that the uncertainties are independent as viewed from the perspective of the motion planner. The assumption is common in state-of-the-art perception modules, \h{for example}~\cite{chai_multipath_2019}, \cite{deo_multi-modal_2018}, \cite{kooij_context-based_2019}.} 

\h{Under the, possibly unbounded, uncertainty of the dynamic obstacles, we constrain the marginal probability of collision at each time step of the trajectory using \textit{chance constraints}, similarly to~\cite{zhu_chance-constrained_2019,schildbach_randomized_2013}. Each chance constraint is subject to an acceptable risk level $\epsilon_k$, which can be tuned accordingly. This implies that we cannot give a non-conservative bound on the collision risk of the full motion plan. However, by frequently recomputing the motion plan, for example in an MPC framework, the actions in the near future are probabilistically safe and risk in later stages is reconsidered when the robot moves closer.
We formulate the motion planning problem as follows:}\h{
\begin{subequations}
\label{eq:motion_planning_problem_general}
\begin{align}
    \min_{\b{u} \in \mathbb{U}} \qquad & \sum_{k = 1}^N J(\b{x}_k, \b{u}_k)\\
  \textrm{s.t.} \qquad & \b{x}_{k + 1} = f(\b{x}_k, \b{u}_k), \ \hu{\b{x} \in \mathbb{X}} \\
  &\hu{\mathbb{P}}\hu{_k} \left[||\b{x}^d_k - \bm{\delta}^v_k||_2 > r, \forall d, v \right] \geq 1 - \epsilon_k, \ \forall k,\label{eq:cc_marginal_probability}
\end{align}
\end{subequations}}
where $\b{u} = \{\b{u}_1, \hdots, \b{u}_N\} \in \mathbb{U}$ are the optimized system inputs subject to input constraints, $\bm{\delta}^v_k \in \Delta_k^v$ is the uncertain position of obstacle $v$ at stage $k$ and $J(\b{x}_k, \b{u}_k) \geq 0$ is the cost function specifying performance metrics. The radius $r$ is the sum of vehicle and obstacle radii. To simplify the notation, we assume this radius to be a constant. The chance constraint, \eqref{eq:cc_marginal_probability}, constrains the probability of collisions between each collision circle $d$ of the vehicle and the collision circle of each dynamic obstacle $v$ \h{at prediction step $k$} to be below the risk level $\epsilon_{\h{k}}$, as visualized in Fig.~\ref{fig:ccp}. \hu{The probability measure $\hu{\mathbb{P}}\hu{_k}$ refers to the modeled uncertainty.}

Problem~\eqref{eq:motion_planning_problem_general} is a chance constrained optimization problem. As discussed in Section~\ref{sec:proposed_approach}, to solve this problem, we rely on the nonconvex scenario optimization (NSO) framework of~\cite{campi_general_2018}, for which we provide an overview in the following section. \hu{This framework can in general provide a bound on the risk with respect to the unknown probability distribution $\mathbb{P}_{\textrm{real}}$, by sampling from the real system. In the real-time setting of this paper, however, collecting samples online is intractable. Instead we propose to sample from the model distribution $\mathbb{P}$, as defined in Assumption \ref{as:model}. For consistency of notation, the results of \cite{campi_general_2018} are presented here using the model $\mathbb{P}$.}


%% file: content/preliminary.tex
\section{NONCONVEX SCENARIO OPTIMIZATION}
\label{sec:preliminaries}
The NSO framework allows us to replace chance constraints with deterministic constraints by sampling. 
Consider the Chance Constrained Problem (\textbf{CCP})
\begin{subequations}
\label{eq:ccp}
\begin{align}
    \min_{\b{u} \in \mathbb{U}} \qquad & J(\b{u})\\
  \textrm{s.t.} \qquad & \mathbb{P}\left[g(\b{u}, \bm{\delta}) \leq 0\right] \geq 1 - \epsilon, \  \bm{\delta} \in \Delta,\label{eq:og_ccp}
\end{align}
\end{subequations}
where $\b{u}$ are decision variables, $\bm{\delta} \in \Delta$ is the realization of the uncertainty and the function \h{$g:\mathbb{X}\times \Delta \rightarrow \mathbb{R}$} is a nonlinear function associated with the nonconvex constraint $g(\b{x}, \bm{\delta})\leq 0$. The authors of \cite{campi_general_2018} established a link between CCP \eqref{eq:ccp} and the deterministic Scenario Program (\textbf{SP}):
\begin{subequations}
\begin{align} 
    \min_{\b{u} \in \mathbb{U}} \qquad & J(\b{u})\\
    \qquad \textrm{s.t.} \qquad & g(\b{u}, \bm{\delta}^i) \leq 0, \ \bm{\delta}^i \in \Delta, \ \forall i \in \mathcal{S}.\label{eq:sp_constraint}
\end{align}
\label{eq:scenario_full}\label{eq:sp}
\end{subequations}
\h{We denote its solution by $\b{u}_{SP}^*$.} Each of the $S$ constraints in \eqref{eq:sp_constraint} is constructed by drawing a sample \h{$\bm{\delta}^i$ from $\Delta$,} and formulating the constraint $g(\b{u}, \bm{\delta}^i)\leq 0$ in the scenario where the sample $\bm{\delta}^i$ is a realization of the uncertainty. Since each of the samples specifies a scenario, the samples themselves are called \textit{scenarios} and the constraints \eqref{eq:sp_constraint} are known as \textit{scenario constraints}. \h{The \textit{violation probability}, $V : \mathbb{U} \rightarrow [0, 1]$, given by
\begin{equation}
    V(\b{u}) \coloneqq \mathbb{P}\left[ \bm{\delta} \in \Delta : g(\b{u}, \bm{\delta}) > \b{0} \right],
\end{equation}}
\h{defines the probability that input $\b{u}$ \hu{violates} a newly observed scenario. The solution of the SP in \eqref{eq:sp} depends on randomly sampled scenarios and hence its violation probability is a random variable \hu{over the product probability measure, given by $\mathbb{P}^{\textrm{S}}$ = $\mathbb{P} \times \hdots \times \mathbb{P}$ (S times)}. To link the SP of \eqref{eq:sp} with the CCP of \eqref{eq:ccp}, we are therefore interested in bounding the probability that $V(\b{u}_{SP}^*)$ satisfies our risk bound $\epsilon$\hu{, a probability which we refer to as the confidence}.}
A key definition in this direction is the \textit{support subsample}.

\smallskip
\noindent\textbf{Definition~\cite{campi_general_2018}: }A \textit{support subsample} of an SP is a subset of scenarios $\mathcal{S}_{\textrm{support}} \subseteq \mathcal{S}$ that results in the same optimizer as the original SP. The cardinality of the support subsample, that is, the support subsample size, is denoted by \h{$s$. The smallest support subsample size is denoted by $s^*$.}

\smallskip
\noindent Theorem 1 in \cite{campi_general_2018} \h{provides the following \hu{confidence bound}}
\begin{equation}
   \h{\mathbb{P}^{\textrm{S}}[V(\b{u}_{SP}^*) > \epsilon(s^*)]} \leq \sum_{s = 0}^{S - 1} {S \choose s} \left[1 - \epsilon(s)\right]^{S- s} = \beta. \label{eq:nonconvex_relation}
\end{equation}
Here \h{$\epsilon(s) : \{0, \hdots, S\}\rightarrow [0, 1]$} can be designed subject to \eqref{eq:nonconvex_relation} and $\epsilon(S) = 1$, an example can be found in~\cite[Sec.~II]{campi_general_2018}. 
Equation \eqref{eq:nonconvex_relation} theoretically links the sampling size $S$, 
\hu{confidence parameter} $\beta$ \hu{(complement of the confidence)} and risk $\epsilon$, based on the observed support sample size. 
\hu{Notice that in this work, as a consequence of using model distribution $\mathbb{P}$, the bound \eqref{eq:nonconvex_relation} applies to the modeled uncertainty rather than the real robot, in contrast with \cite[Th. 1]{campi_general_2018}.} 



%% file: content/method.tex
\begin{figure}[t]
    \centering
    \figuretag{1}
     \begin{subfigure}[b]{0.21\textwidth}
         \centering
         \includegraphics[width=\textwidth]{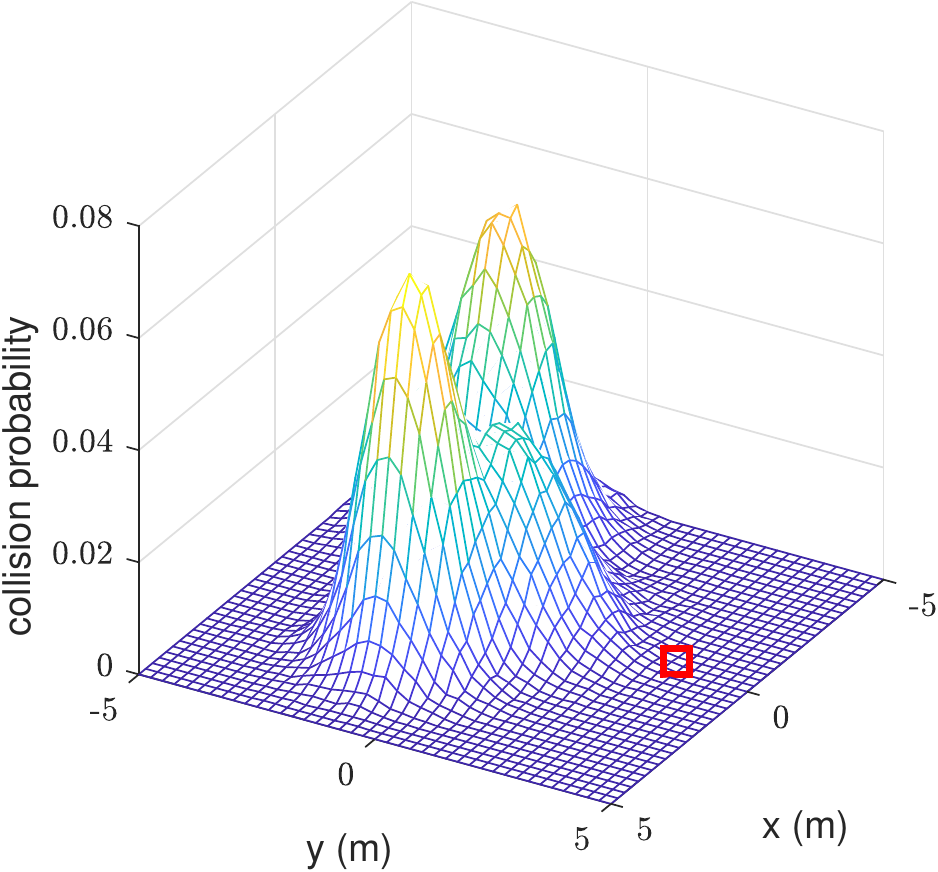}
         \caption{\h{Quadratic CCP \eqref{eq:cc_marginal_probability}}}
         \label{fig:ccp_quad_real}
     \end{subfigure}
     \
     \begin{subfigure}[b]{0.21\textwidth}
         \centering
         \includegraphics[width=\textwidth]{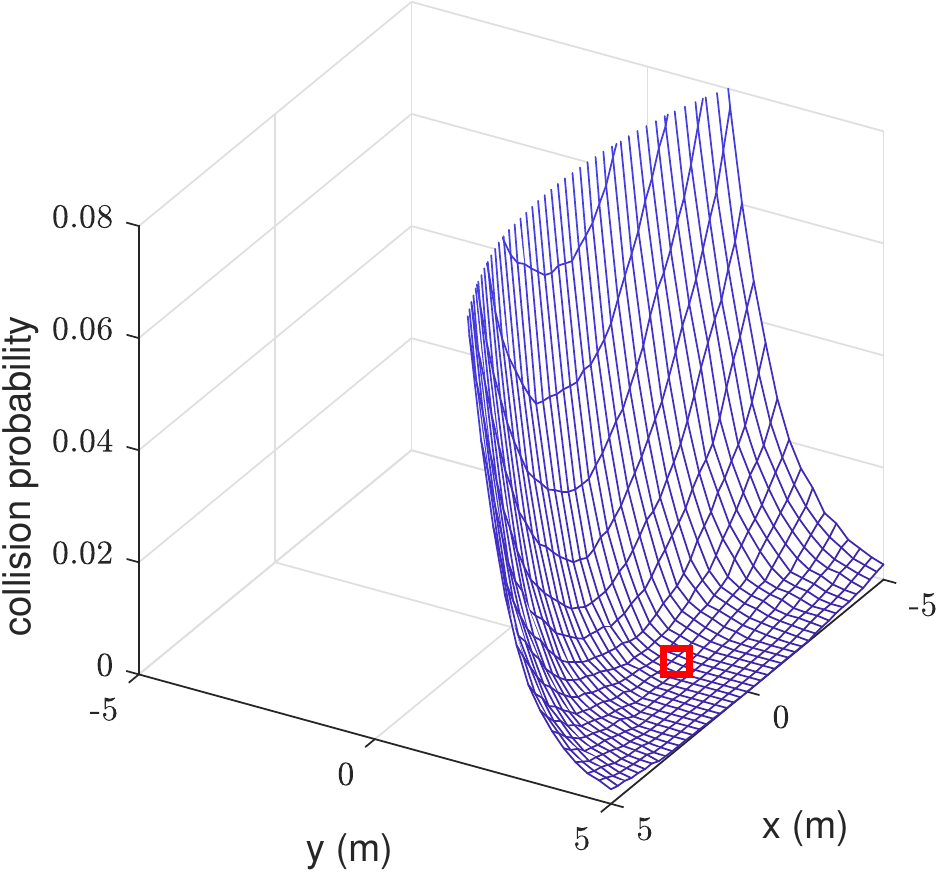}
         \caption{\h{Linearized CCP \eqref{eq:cc_linearized}}}
         \label{fig:ccp_quad}
     \end{subfigure}
        \caption{\h{Grid wise evaluation of the collision probability, with $r = 0.5$~m of chance constraints \eqref{eq:cc_marginal_probability} and \eqref{eq:cc_linearized} 
        in an example where $\bm{\delta}$ follows a Mixture-of-Gaussians (MoG) distribution. The red square denotes the linearization point.}}
        \label{fig:linearized_ccp}
\end{figure}
\begin{figure*}[h]
    \figuretag{2}
     \centering
     \begin{subfigure}[b]{0.19\textwidth}
         \centering
         \includegraphics[width=\textwidth]{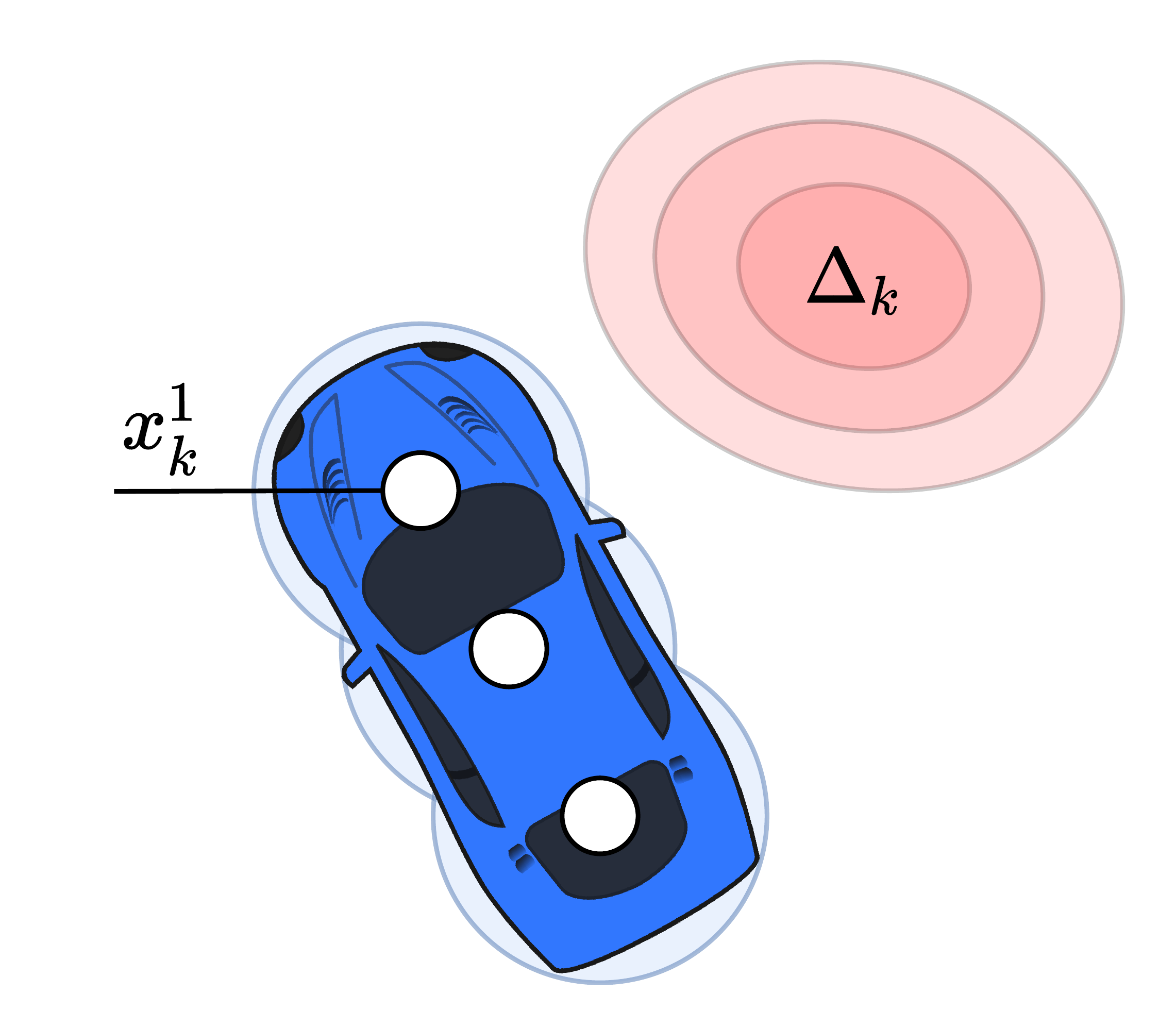}
         \caption{Original CCP \eqref{eq:cc_marginal_probability}}
         \label{fig:ccp}
     \end{subfigure}
     \
     \begin{subfigure}[b]{0.19\textwidth}
         \centering
         \includegraphics[width=\textwidth,trim = 0 0 0 104, clip]{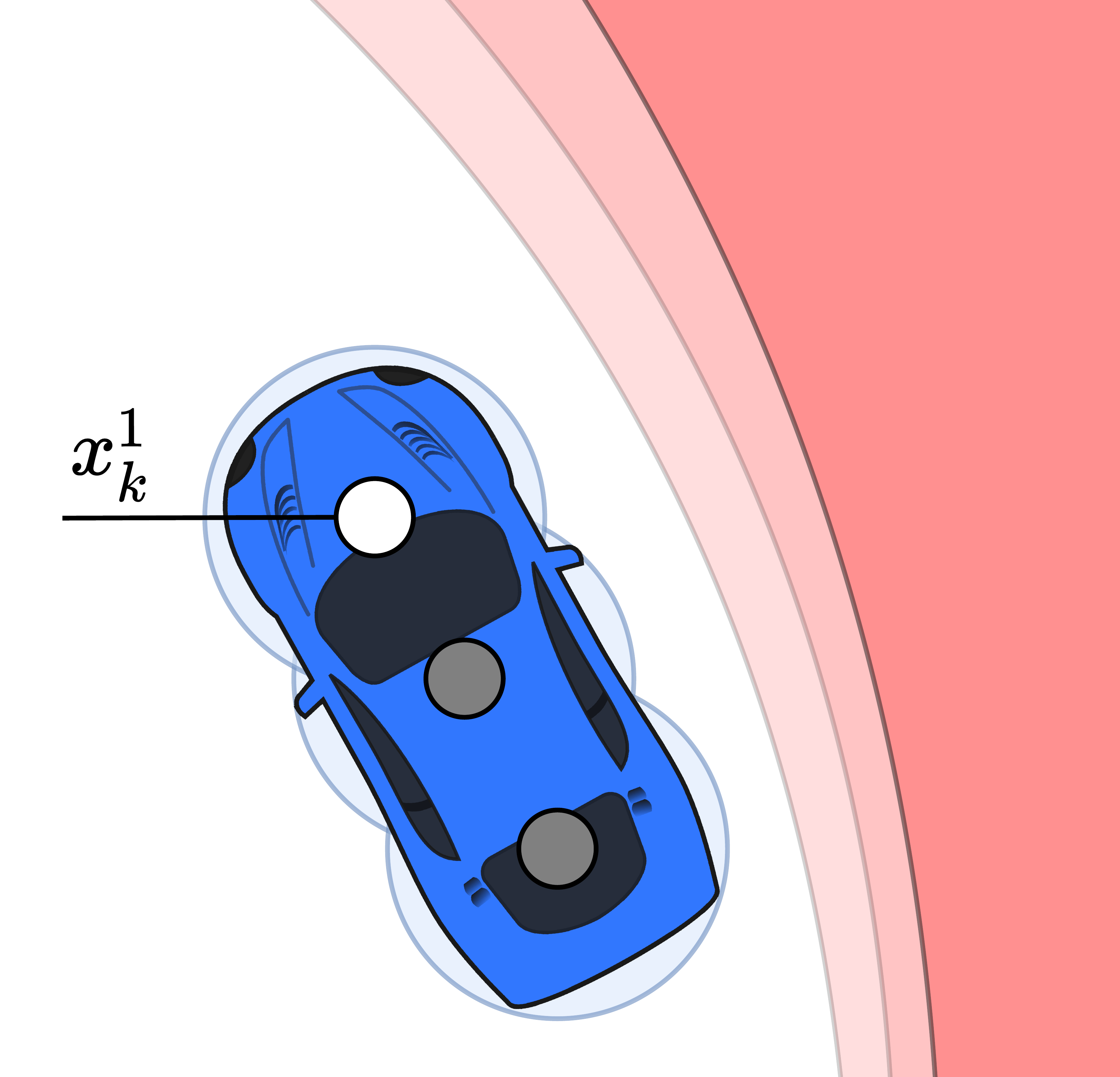}
         \caption{\h{Linearized CCP \eqref{eq:cc_linearized}}}
         \label{fig:linear_ccp}
     \end{subfigure}
     \
     \begin{subfigure}[b]{0.19\textwidth}
         \centering
         \includegraphics[width=\textwidth]{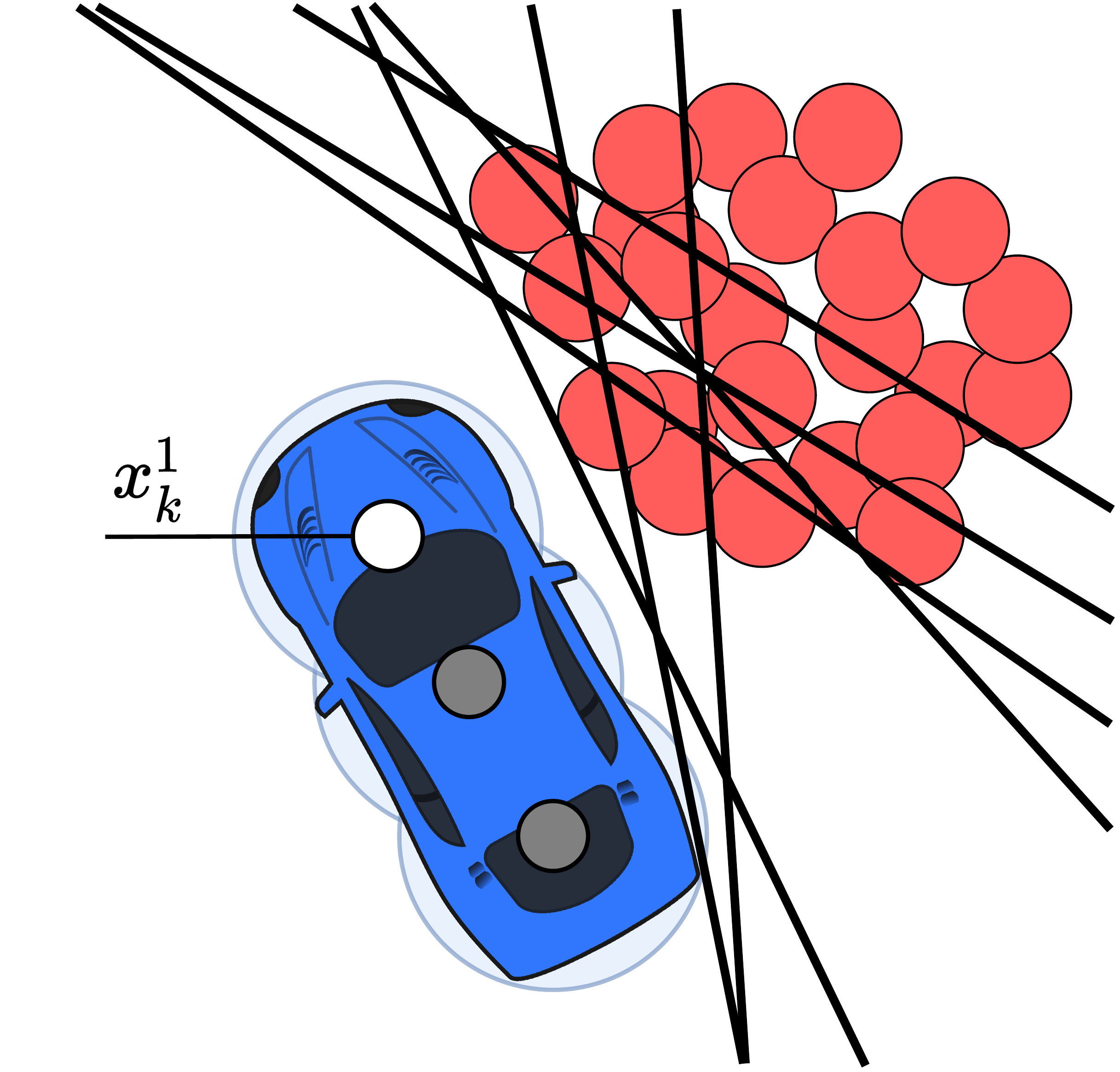}
         \caption{\h{Linear} SP \eqref{eq:scenario_halfspaces}}
         \label{fig:sp_lin}
     \end{subfigure}
     \
     \begin{subfigure}[b]{0.19\textwidth}
         \centering
         \includegraphics[width=\textwidth]{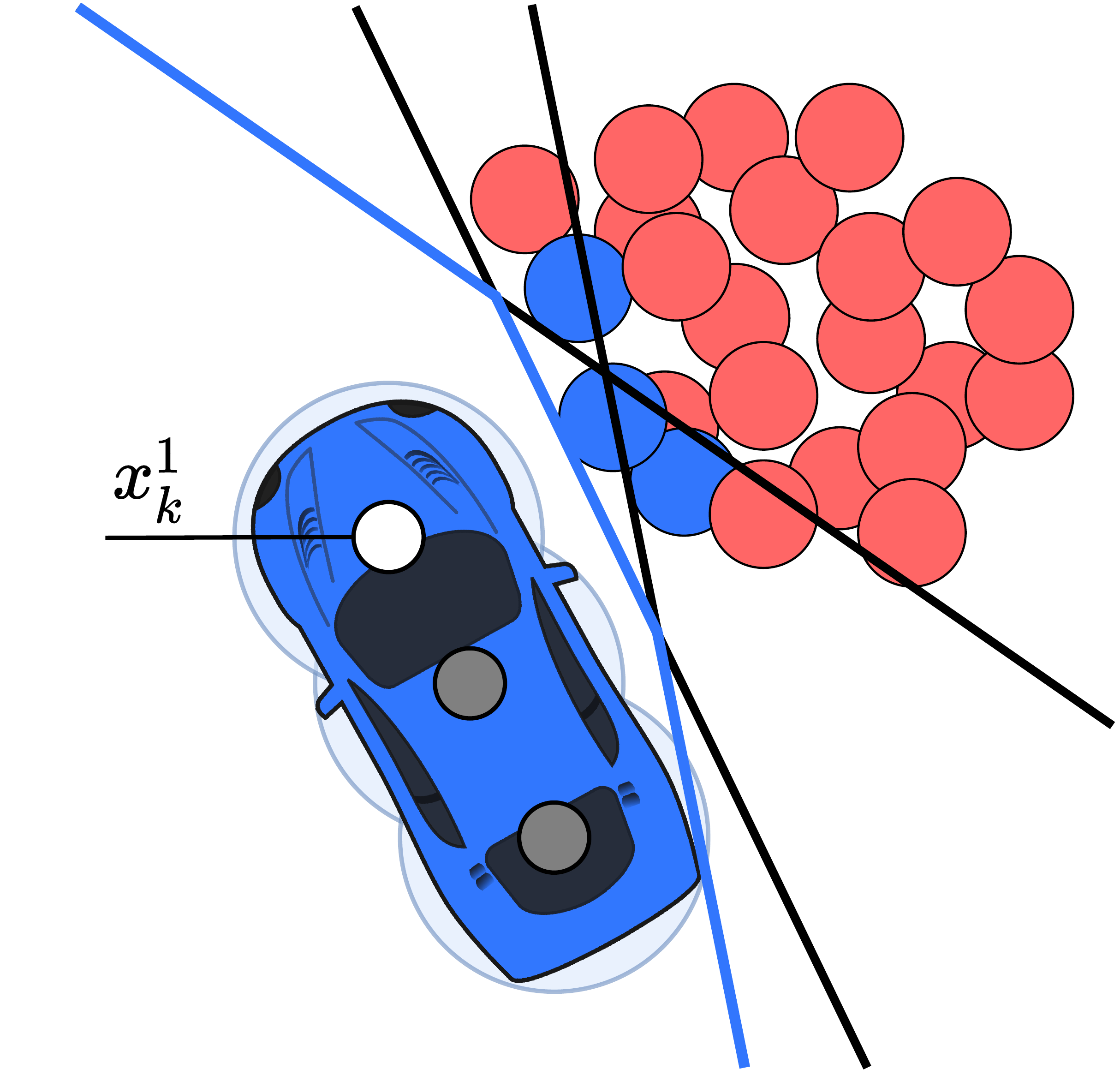}
         \caption{Pruned SP \eqref{eq:final_constraints}}
         \label{fig:sp_poly}
     \end{subfigure}
        \caption{Our approach exemplified for one robot's disc and one dynamic obstacle for a single stage. The robot and the obstacle are drawn in blue and red, respectively. Fig. \ref{fig:ccp} shows the 1 $\sigma$ to 3 $\sigma$ interval of the uncertainty in red shades. Fig. \ref{fig:linear_ccp} shows \h{the probabilistic collision region when linearized from the robot disc at the front. Fig. \ref{fig:sp_lin} shows the sampled locations in red and boundaries of the constraints in black. Fig. \ref{fig:sp_poly} shows the resulting minimal polytope in blue.}}
        \label{fig:removal_gaussian}
\end{figure*}
\setcounter{figure}{2}

\section{PROPOSED APPROACH}
\label{sec:proposed_approach}
Our method relies on the Model Predictive Contouring Control (MPCC) framework~\cite{brito_model_2019} to define the objective to optimize to plan a suitable path for the robots. Our method differs from~\cite{brito_model_2019} in the way we deal with dynamic obstacles, as detailed in the rest of the section. As such we will refer to our approach as Scenario-MPCC (S-MPCC).
To present the method, we consider a single dynamic obstacle and one of the discs used to represent the vehicle\footnote{Section~\ref{subsec:multiple_vrus} shows how this case extends linearly to multiple dynamic obstacles and multiple discs.}.
\subsection{\hu{Chance Constraints Linearized in the Robot Position}}
\label{subsec:linearize_collision_regions}
\h{Chance constraints \eqref{eq:cc_marginal_probability} are nonconvex in the robot \hu{position} when sampled (see discs in Fig.~\ref{fig:sp_lin}) \hu{and the associated SP may have many local optima and a sizable support subsample. We therefore consider a linearization of the collision regions (depicted by the lines in Fig.~\ref{fig:sp_lin}) before sampling to decrease the support subsample size of the SP. This step reduces the risk of its solution significantly. We modify the constraints as}
\begin{subequations}
\label{eq:hyperplane_definition}
\begin{align}
    &\b{A}_k = \frac{\bm{\delta}_k - \hat{\b{x}}_{k}}{||\bm{\delta}_k - \hat{\b{x}}_{k}||}, \qquad b_k = {\b{A}_k}^T(\bm{\delta}_k - \b{A}_k r),\label{eq:A_and_b}\\
   &\hu{\mathbb{P}}\hu{_k}\left[\b{A}_k^T\b{x}_k \leq b_k\right] \geq 1 - \epsilon_k, \forall k, \bm{\delta}_k \in \Delta_k, \label{eq:cc_linearized}
\end{align}
\end{subequations}
where we linearize the collision region with respect to $\hat{\b{x}}_{k}$, the $k$-step ahead prediction of the robot \hu{position}. We employ the trajectory of the previous planning cycle, forward propagated, as predictor. That is\hu{\footnote{\h{We denote by $\b{x}_{t|k}$ the $k$-step ahead prediction of the robot trajectory for the MPC planning cycle at time $t$}}}, $\hat{\b{x}}_{t|k} = \b{x}_{t-1|k+1}$ and $\hat{\b{x}}_{t|N} = \b{x}_{t-1|N}$. Hence, we search for collision-free solutions around the planned trajectory of the previous planning cycle. \hu{We show in Sec. \ref{subsec:free_space} that after linearization, the free-space of the resulting SP is convex in the robot position.} A comparison between chance constraints \eqref{eq:cc_marginal_probability} and \eqref{eq:cc_linearized} for an example is provided in Fig.~\ref{fig:linearized_ccp}. The linearized chance constraints capture less of the shape of the distribution, but are accurate near $\hat{\b{x}}_k$ and thus sufficient for motion planning. Note that the linearizations are performed for each stage of the trajectory, as illustrated in Fig.~\ref{fig:linear_ccp}.}

\subsection{Scenario Program}
\label{subsec:scenario_program}

For each of the chance constraints in \eqref{eq:cc_linearized} we construct a set of deterministic constraints by sampling from the uncertainty. The red circles in Fig.~\ref{fig:sp_lin} represent these samples \h{and the black lines are the \emph{scenarios} (Sec.~\ref{sec:preliminaries}).} The resulting SP is given by
\hu{\begin{subequations}
\begin{align}
 \min_{\b{u} \in \mathbb{U}} \qquad & \sum_{k = 1}^N J(\b{x}_k, \b{u}_k)\\
  \textrm{s.t.} \qquad & \b{x}_{k + 1} = f(\b{x}_k, \b{u}_k), \ \hu{\b{x} \in \mathbb{X}} \\
  &\b{A}_k^T(\bm{\delta}^i_k, \hat{\b{x}}_k)\b{x}_k\!\leq\!b_k(\bm{\delta}^i_k, \hat{\b{x}}_k), \forall i \in \mathcal{S}_k, \ \forall k.\label{eq:scenario_halfspaces_constraints}
\end{align}
\label{eq:scenario_halfspaces}
\end{subequations}}
The theoretic properties of SPs, discussed in Sec.~\ref{sec:preliminaries}, are limited to CCPs with one chance constraint. However, \eqref{eq:cc_linearized} describes multiple chance constraints, one for every stage of the planned trajectory. We now show that multiple chance constraints can be handled separately, resulting in a probabilistic feasibility property per stage.
\begin{theorem}\label{theory:per_stage} 
\h{Under Assumption \ref{as:independent},} the probability that the solution \hu{of SP \eqref{eq:scenario_halfspaces}} violates its associated chance constraint at stage $k$, satisfies
\begin{equation}
    \hu{\mathbb{P}^{{\hu{\textrm{S}}}_k}_k}[V_k(\b{u}_{SP}^*) > \epsilon_k(s_k^*)]\leq \beta_k(\h{S_k}),
\end{equation}
where
\begin{equation}
    \beta_k(\h{S_k}):= \sum_{s = 0}^{S_k - 1} {S_k \choose s}\left[1 - \epsilon_k(s)\right]^{S_k - s}.\label{eq:guarantee} 
\end{equation}
\end{theorem}

\begin{proof}[Proof of Th. \ref{theory:per_stage}]
The proof follows along the lines of the convex proof \cite[Th. 4.1]{schildbach_randomized_2013}. In the following, we derive the result for $k=1$. The proof is analogous for all other $k$. We use the notation $\bm{\omega}_k = \{\bm{\delta}_k^1, \hdots, \bm{\delta}_k^{S_k}\}$ to denote the collection of all samples per stage. Consider the \hu{complement of the confidence} of the first stage, when the samples of all other stages have been drawn,
\begin{equation}
    \hu{\mathbb{P}^{{\hu{\textrm{S}}}_1}_1}[V_1(\b{u}_{SP}^*(\bm{\omega}_1)) > \epsilon_1(s_1^*)\ | \ \bm{\omega}_2, \hdots \bm{\omega}_N], \ \bm{\omega}_1 \in \Delta^{\textrm{S}_1}.
\end{equation}
\h{Under Assumption \ref{as:independent}, the samples $\bm{\omega}_1$ are drawn independently from the samples $\bm{\omega}_2, \hdots, \bm{\omega}_N$}. Moreover, since $\bm{\omega}_2, \hdots, \bm{\omega}_N$ \hu{have been observed}, we can merge their respective constraints into the feasible set
\begin{equation}
    \Tilde{\mathbb{X}}^{2:N} =  \h{\prod_{k = 2}^N} \{\b{x}_k \ | \ g(\b{x}_k, \bm{\omega}_k) \leq 0 \}.
\end{equation}
This results in the following modified optimisation problem
\begin{subequations}
\label{eq:stage_1}
\begin{align} 
    \min_{\b{u} \in \mathbb{U}} \qquad & \sum_{k = 1}^N J(\b{x}_k, \b{u}_k)\\
  \textrm{s.t.} \qquad & \b{x}_{k + 1} = f(\b{x}_k, \b{u}_k), \ \hu{\b{x} \in \mathbb{X}} \\
   &\mathbb{P}\hu{_1}[\h{g(\b{x}_1, \bm{\delta}_1) \leq 0}] \geq 1 - \epsilon_1 \\
   &\b{x}_{2:N} \in \Tilde{\mathbb{X}}^{2:N}.\label{eq:ccp_multistage_1}
\end{align}
\end{subequations}
This problem is a nonconvex CCP of the form \eqref{eq:ccp} with one chance constraint, hence we can apply \eqref{eq:nonconvex_relation}, which shows that the \hu{confidence} of the first stage satisfies the proposed theorem for $k = 1$ and analogous derivations apply for $k=2, \hdots, N$. Even though constraints \eqref{eq:ccp_multistage_1} are deterministic, the solution to the optimization problem has not changed compared to \hu{SP \eqref{eq:scenario_halfspaces}}. We therefore conclude that the result holds.
\end{proof}

\subsection{\h{Probabilistic Safety Guarantees}}
\label{subsec:free_space}
\h{The key insight that makes our approach tractable is that due to the geometric structure of the problem, the free space may be described by only a small subset of the scenarios. To see this, first note that each scenario constraint in \eqref{eq:scenario_halfspaces_constraints} defines a half-space. The collision-free space, if it exists, is formed by the intersection of half-spaces and is convex, as i) each half-space is convex and ii) the intersection of convex constraints is convex. This results in a free space polytope $\mathcal{P}_k$ (see Fig.~\ref{fig:sp_poly}), \hu{spanned} by those half-spaces that form the boundary of the polytope. We may define this subset of half-spaces by their indices as
\begin{equation*}
    \mathcal{H}_k \coloneqq \{i \ | \ \exists \b{x}_k \in \mathcal{P}_k, \b{A}_k^T(\bm{\delta}^i_k, \hat{\b{x}}_k)\b{x}_k = b_k(\bm{\delta}^i_k, \hat{\b{x}}_k) \}.
\end{equation*}
The usefulness of the set $\mathcal{H}_k$ is twofold. First, we may replace \eqref{eq:scenario_halfspaces_constraints} with only those half-spaces that span polytope $\mathcal{P}_k$, greatly reducing the size of the online optimization problem. Second, the set $\mathcal{H}_k$ contains indices of the constraints that may be active during optimization and hence the support subsample is bounded by its cardinality, that is, $s_k^* \leq |\mathcal{H}_k|$. We use the latter fact to establish the link between the CCP \hu{subject to \eqref{eq:cc_linearized} and SP \eqref{eq:scenario_halfspaces}}. There always exists an upper bound, $\bar{s}$, for the cardinality of $\mathcal{H}_k$ and for our problem we find experimentally that this upper bound $\bar{s}$ is much smaller than the sample size. That is, for uncertainty distributions where the samples are not cluttered at the boundary, only few scenarios are active.}

We can now compute sampling size $S_k$ offline, using \emph{(i)} Theorem \ref{theory:per_stage}, \emph{(ii)} upper bound $\bar{s}$, \emph{(iii)} \hu{confidence parameter} $\beta_k$, and \emph{(iv)} risk $\epsilon_k$. The \hu{SP} we solve online is given by:
\begin{subequations}
\label{eq:scenario_final}
\begin{align} 
    \min_{\b{u} \in \mathbb{U}} \qquad & \sum_{k = 1}^N J(\b{x}_k, \b{u}_k)\\
  \textrm{s.t.} \qquad & \b{x}_{k + 1} = f(\b{x}_k, \b{u}_k), \ \hu{\b{x} \in \mathbb{X}} \\
  \qquad \qquad & \h{\b{A}^T_k(\bm{\delta}^i_k, \hat{\b{x}}_k)\b{x}_k \leq b_k(\bm{\delta}^i_k, \hat{\b{x}}_k), \ i \in \mathcal{H}_k}.\label{eq:final_constraints}
\end{align}
\end{subequations}
Algorithm \ref{alg:scenario_vru} summarizes our method. Online, we sample from the distribution and identify the minimal \h{polytope} and the support subsample size (line 4-9). We then solve optimization problem \eqref{eq:scenario_final} (line 10) and use the first input as control input (line 11). 
\hu{In the following we provide a result for improving performance by discarding outlier scenarios.}

\begin{algorithm}[t]
\caption{S-MPCC}
\label{alg:scenario_vru}
\begin{algorithmic}[1]
\small{
\STATE Compute $S_k$ from $\epsilon_k$, $\bar{s}$, for all $k$

\FORALL{$t = 1, 2, \hdots$}
    \STATE  $\Delta_k^t \leftarrow $ Retrieve uncertainty from perception module
    
    \FORALL{$k = 1, \hdots, N$}
        
        \STATE Sample $\bm{\delta}_k^i \in \Delta_k^t$, $i = \{1, \hdots, S_k\}$
        \STATE \h{Compute $\b{A}_k^i, b_k^i$ from \eqref{eq:A_and_b} for all samples}
        \STATE \h{Find $\mathcal{H}_k$ and verify $|\mathcal{H}_k| \leq \bar{s}$}
    
    \ENDFOR
    \STATE $\b{u}\h{_t} \leftarrow $ Solve \eqref{eq:scenario_final}
    
    \STATE \textbf{Output: }\h{$\b{u}_{t|1}$}
   
\ENDFOR}

\end{algorithmic}
\end{algorithm}

\begin{theorem}
Consider solving the CCP \eqref{eq:ccp} using the SP \eqref{eq:sp}, where after sampling, part of the scenarios are discarded. Suppose that we have a discarding algorithm $\mathcal{R}$ that removes $R$ of the $S$ scenarios, leaving $P=S-R$ scenarios to be considered for the optimization. Let $\epsilon(s)$ be a function such that $\epsilon(P) = 1$ and
\begin{equation*}
    \beta(\h{S, P}) = {S \choose P}\sum_{s=0}^{P-1}{P \choose s}[1- \epsilon(s)]^{P-s}.
\end{equation*}
Then the probability that the solution of the SP \eqref{eq:sp} is \h{in}feasible for the original CCP \eqref{eq:ccp} satisfies the upper bound
\begin{equation}
    \h{\mathbb{P}^{\hu{\textrm{S}}}[V(\b{u}_{SP}^*) > \epsilon(s^*)]} \leq \beta(\h{S, P}). 
\end{equation}
\end{theorem}

\begin{proof}
Consider the partitioning of the probability space:
\begin{equation}
    \Delta^S_{\mathcal{I}_p} = \{\bm{\delta}^S \in \Delta^S \ | \ \mathcal{R}(\bm{\delta}^S) = \mathcal{I}_p\},
\end{equation}
The sets $\Delta^S_{\mathcal{I}_p}$ are events where the picking algorithm selected the indices $\mathcal{I}_p$. \hu{Define the set where the risk bound is violated}
\begin{equation}
    \mathcal{B}_{\hu{\mathcal{I}_p}} = \{\bm{\delta}^S \!\ | \mathcal{R}(\bm{\delta}^S) = \mathcal{I}_p, V(\b{\h{u}}_{\mathcal{I}_p}^*) \!> \!\epsilon(s_{\mathcal{I}_p}^*) \}.
\end{equation}
Notice that the last condition is upper bounded by \eqref{eq:nonconvex_relation} with $S=P$. But the distribution of the samples is biased due to the samples that were removed from the iid sample set. We obtain the following bound on the biased sample set
\begin{equation}
    \mathbb{P}^{\hu{\textrm{S}}}[\mathcal{B}_{\hu{\mathcal{I}_p}}] \leq \sum_{s=0}^{P-1}{P \choose s}[1- \epsilon(s)]^{P-s}.
\end{equation}
This result holds for all index sets which also contains all the possible biases introduced by $\mathcal{R}$.  \hu{Hence, the upper bound}
\begin{align}
    \mathbb{P}^{\hu{\textrm{S}}}[\mathcal{B}] &= \hu{\mathbb{P}^{\hu{\textrm{S}}}\bigg[\bigcup_{\mathcal{I}_p} \mathcal{B}_{\mathcal{I}_p} \bigg]\leq \hu{\beta(S, P)},}\label{eq:discarding_relation}
\end{align}
\hu{is attained by independence of the samples.}
\end{proof}
\begin{remark}
Bound \hu{\eqref{eq:discarding_relation}} is conservative. For example if we pick a random discarding algorithm for $\mathcal{R}$, then the samples are still iid and we can use \eqref{eq:nonconvex_relation} \hu{directly with $S=P$, giving}
\begin{equation*}
    \sum_{s=0}^{P-1}{P \choose s}[1- \epsilon(s)]^{P-s} = \beta,
\end{equation*}
which is generally much tighter than \eqref{eq:discarding_relation}. However, even if the bound is conservative we can use it to remove extreme scenarios, leading to generally better performance.
\end{remark}

\subsection{Multiple Dynamic Obstacles and Discs}
\label{subsec:multiple_vrus}
To apply the strategy above to more than one obstacle, we use the fact that scenario optimization is distribution agnostic. We combine the predictions of the obstacles into \h{a probability space $\Delta_k = \begin{bmatrix} \Delta^0_k & \hdots & \Delta^V_k \end{bmatrix}^T$, where samples are denoted $\bm{\delta}_k = \begin{bmatrix} \bm{\delta}^0_k & \hdots & \bm{\delta}^V_k \end{bmatrix}^T$.}
Although the stacked distribution $\bm{\delta}_k$ could be used to model the correlation between the movement of obstacles, we will sample each component separately from individual probability distributions. The chance constraints \eqref{eq:cc_linearized} need to include all obstacles and are modified as follows{\thickmuskip=0.5\thickmuskip
\begin{equation*}
    \hu{\mathbb{P}_k}\left[\b{A}^T_k(\bm{\delta}^v_k, \hat{\b{x}}_k)\b{x}_k \leq b_k(\bm{\delta}^v_k, \hat{\b{x}}_k), \forall v\right] \geq 1 - \epsilon_k, \ \bm{\delta}_k \in \Delta_k, \ \forall k.
\end{equation*}}
The rest of the method follows analogously to the single obstacle approach but where the scenarios are drawn for each obstacle, resulting in more scenarios to process before obtaining the free space polytope. 
In the case of multiple vehicle discs, we formulate multiple chance constraints of the form \eqref{eq:cc_marginal_probability}, one for each collision disc. We apply the method described in this Section per disc as samples for each of the discs are independent.

%% file: content/implementation.tex
\section{S-MPCC WITH GAUSSIAN UNCERTAINTIES}
\label{sec:implementation}
A common class of uncertainties are the (truncated) Gaussian uncertainties. This section presents a detailed formulation of Algorithm \ref{alg:scenario_vru}, namely Algorithm \ref{alg:detailled_alg}, one can use in the case of (truncated) Gaussian uncertainty. 

The first step of Algorithm~\ref{alg:detailled_alg} is to determine the sample size. \h{We set $\epsilon_k = 1 - 0.9889$, equivalent to the probability mass under the $3~\sigma$ interval of a bivariate Gaussian (generally considered as safe). Since the risk has logarithmic dependency on $\beta_k$~\cite{campi_general_2018}, $\beta_k$ is generally small. We pick $\beta_k=1\cdot10^{-6}$, i.e., one in a million SPs may not be feasible for the original CCP\footnote{\h{Note that the designer can choose to keep safety margin in the obstacle radius such that a failure does not have to result in a collision.}}. The removal size $R = 50$ is empirically determined, verifying that outliers are removed. Upper bound $\bar{s}$ is guessed and increased until it is never exceeded in practice. We find $\bar{s} = 20$. Evaluating \eqref{eq:discarding_relation}, we are able to pick $S_k \approx 53050$ (line 1). We note that the main dependency of the sample size is the acceptable risk $\epsilon_k$. Sampling more scenarios results in a higher probability of safety, but at the cost of more conservative trajectories and increased computation times.}
\begin{algorithm}[t]
\caption{Detailed S-MPCC for (truncated) Gaussian}
\label{alg:detailled_alg}
\begin{algorithmic}[1]\small{
\STATE Determine $S_k$ from $\epsilon_k$, $\beta_k$, $\bar{s}$, $R$
\STATE $\b{u}^i \leftarrow \mathbb{U}\times\mathbb{U}, \ \forall i = \{1, \hdots, S_k\}$ (uniform random)
\STATE $z_0^i = \sqrt{-2\ln{u^i_1}}\cos{(2\pi u^i_2)}, \ \forall i = \{1, \hdots, S_k\}$ (BMT)
\STATE $z_1^i = \sqrt{-2\ln{u^i_1}}\sin{(2\pi u^i_2)}, \ \forall i = \{1, \hdots, S_k\}$ (BMT)
\STATE Verify relevance of samples $\b{z}$, prune irrelevant
\FORALL{$t = 1, 2, \hdots$}
    
    \FORALL{$k = 1, \hdots, N$}
        \STATE $\bm{\delta}_k^i \leftarrow$ \eqref{eq:gaussian_tf}
        \STATE $\hat{\b{\delta}}_k^l \leftarrow $ apply $\mathcal{R}$ to closest $R+l$ scenarios in $\bm{\delta}_k^i$ 
        \STATE $\mathcal{P}_k \leftarrow $ intersection algorithm on $\mathcal{H}(\hat{\b{\delta}}_k^l) \bigcup \mathcal{H}_k^{\textrm{range}}$ %
    \ENDFOR
    
    
\ENDFOR
}
\end{algorithmic}
\end{algorithm}

\h{Instead of online sampling, we may sample a set of parameterized samples offline followed by an online transformation. This reduces the online operations, resulting in lower computation times. We describe this approach for the (truncated) Gaussian case.} We generate offline a number of batches with $S_k$ bivariate Gaussian samples, centered at the origin and with $\b{\Sigma} = \b{I}$, where $\b{I}$ is the identity matrix (line 2-4). These samples are obtained using the Box-Muller Transformation (BMT) \cite{muller_note_1958}, which also allows us to draw radially truncated Gaussian samples by simply changing the support domain of $u_1$ to $[e^{-\frac{r^2}{2}}, \ 1]$~\cite{martinet_efficient_2012}. Most of the samples will be in the center of the distribution and will not be relevant online. Hence, we run our online algorithm for scenario selection (explained later), offline and aggregate the set of selected scenarios. Scenarios that are not in this set are pruned offline (line 5). In the $3~\sigma$ example, approximately $95$\% of the scenarios are removed offline.

\h{Online, we are only required to transform the offline samples from the standard bivariate normal distribution to the estimated mean and variance of the uncertainty (line 8), which is computed using}
\begin{equation}
    \b{\delta}_k^i = \b{A}_k^T\b{z}_k^i + \b{\mu}_k, \quad \b{A}_k^T\b{A}_k = \b{\Sigma}_k.\label{eq:gaussian_tf}
\end{equation}
We select for each obstacle only one batch of samples. The obstacle predictions are sampled with that batch for all stages and all time steps. This provides the motion planner with consistent constraints. To further reduce the computational load, we search online only for the $l+R$ scenarios closest to considered vehicle position, where we use $l=150$ in the following experiments, we assume that this set contains the support subsample. We then apply the discarding algorithm $\mathcal{R}$, which removes the $R$ scenarios furthest from the mean of the distribution (line 9). We construct \h{half-spaces} from the remaining $l$ scenario and add four \h{half-spaces} to constrain the vehicle in a square workspace. To find the minimal polygon \h{in 2D} from this set of \h{half-spaces}, we use an intersection based algorithm (line 10). The algorithm explores the intersections in the inner polygon in a counter-clockwise fashion. The lines traveled form the minimal polygon.
In the following simulations and experiments, we incorporate our dynamic obstacle avoidance method in the MPCC framework~\cite{brito_model_2019}. We introduce a cost term that activates when the robot gets close to the boundaries of the free space polygon, to penalize movement close to pedestrians.

%% file: content/experiments.tex
\begin{figure*}[t]

    \centering
    \begin{subfigure}{0.28\linewidth}
        \centering
        \includegraphics[width=\textwidth, trim = 0 0 0 0, clip]{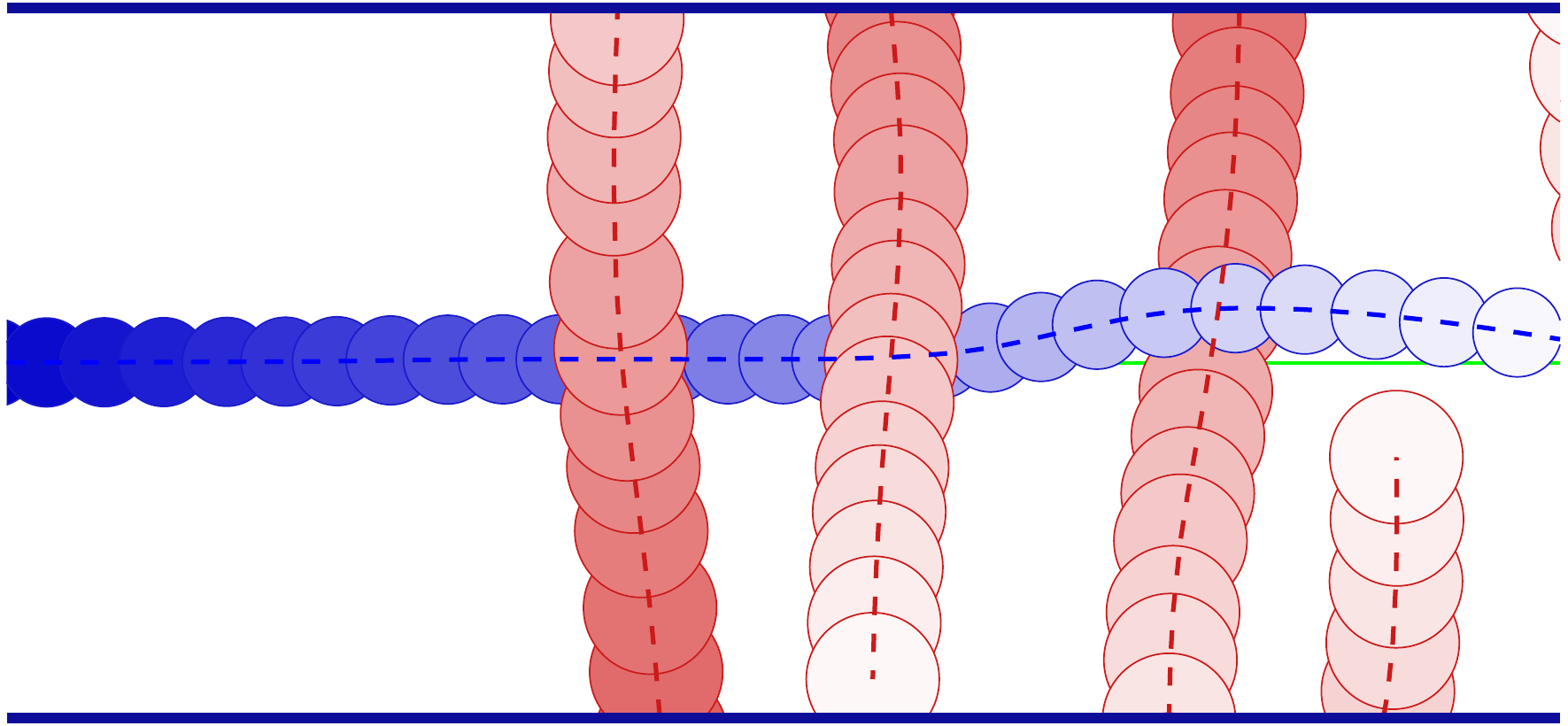}
        
        \vspace{1.5mm}
        
    	\includegraphics[width=\textwidth, trim = 0 0 0 0, clip]{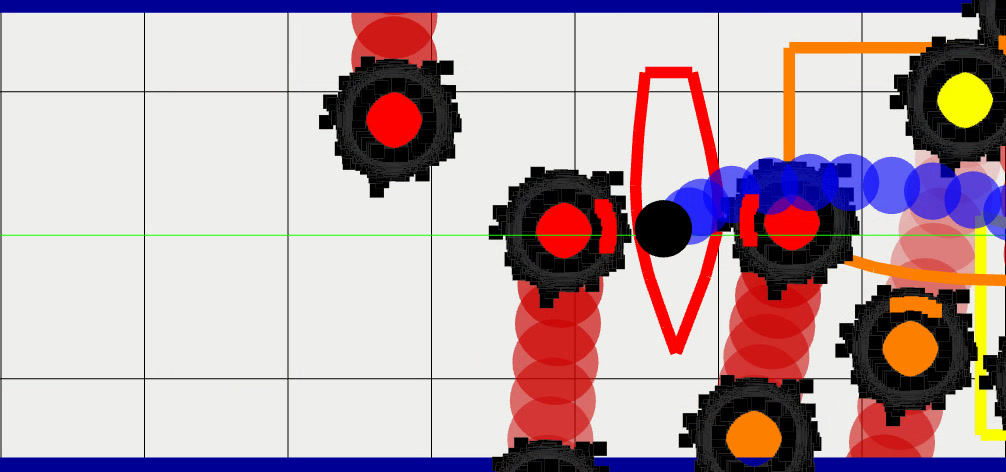}
    	\caption{\h{Gaussian}}
    	\label{fig:sim_gaussian}
    \end{subfigure}
    \hspace{2mm}
    \begin{subfigure}{0.28\linewidth}
        \centering
        \includegraphics[width=\textwidth, trim = 0 0 0 0, clip]{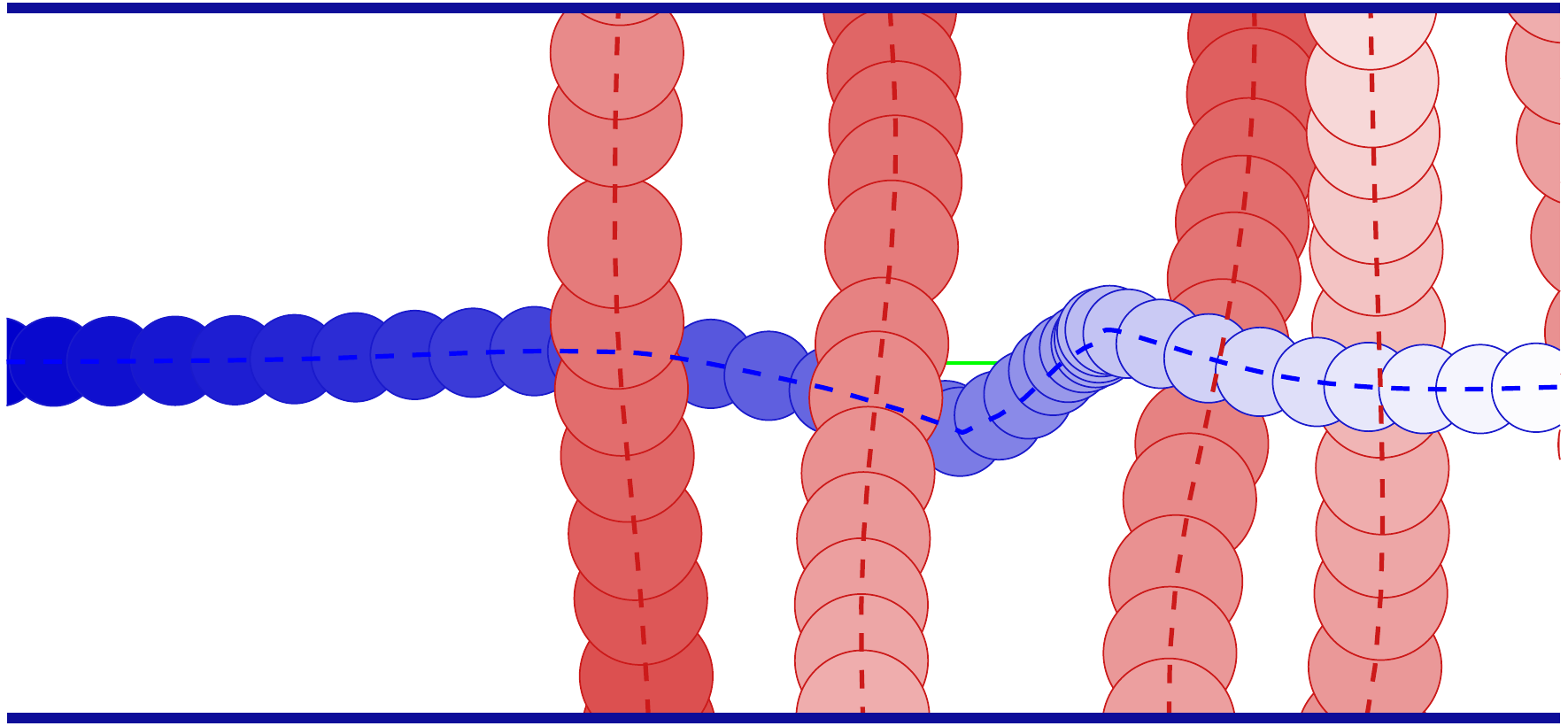}
        
        \vspace{1.5mm}
        
    	\includegraphics[width=\textwidth, trim = 0 0 0 0, clip]{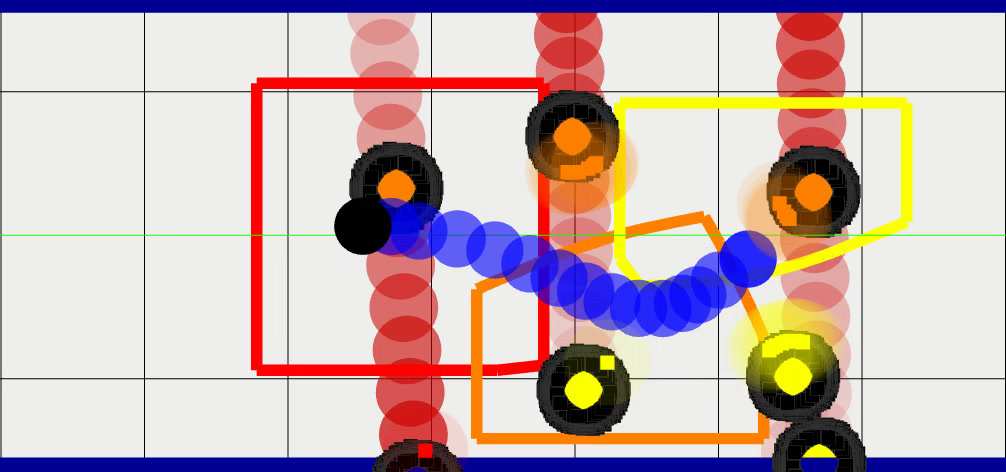}
    	\caption{\h{Radially truncated Gaussian}}
    	\label{fig:sim_nongaussian}
    \end{subfigure}
    \hspace{2mm}
    \hspace{1mm}
    \begin{subfigure}{0.28\linewidth}
        \centering
        \includegraphics[width=\textwidth, trim = 0 0 0 0, clip]{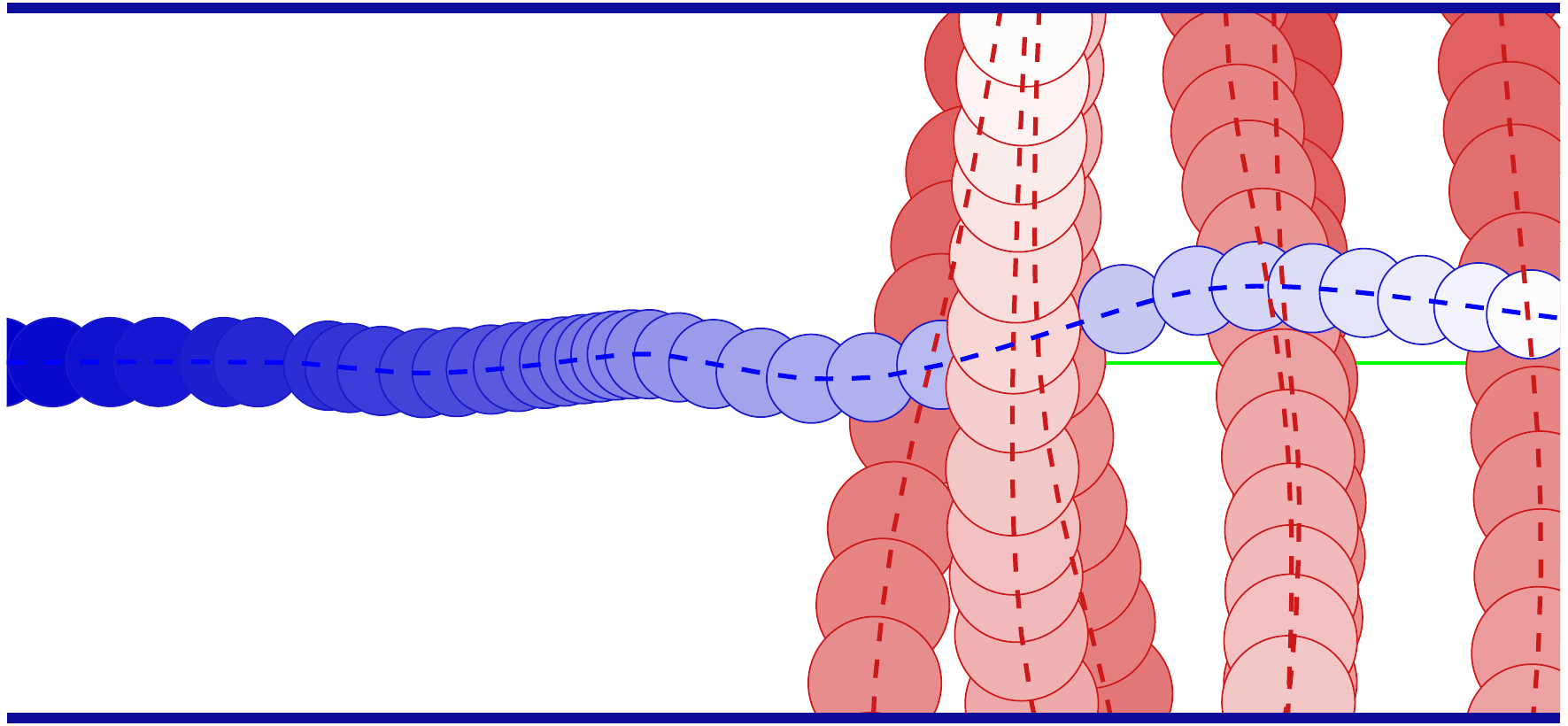}
        
        \vspace{1.5mm}
        
    	\includegraphics[width=\textwidth, trim = 0 0 0 0, clip]{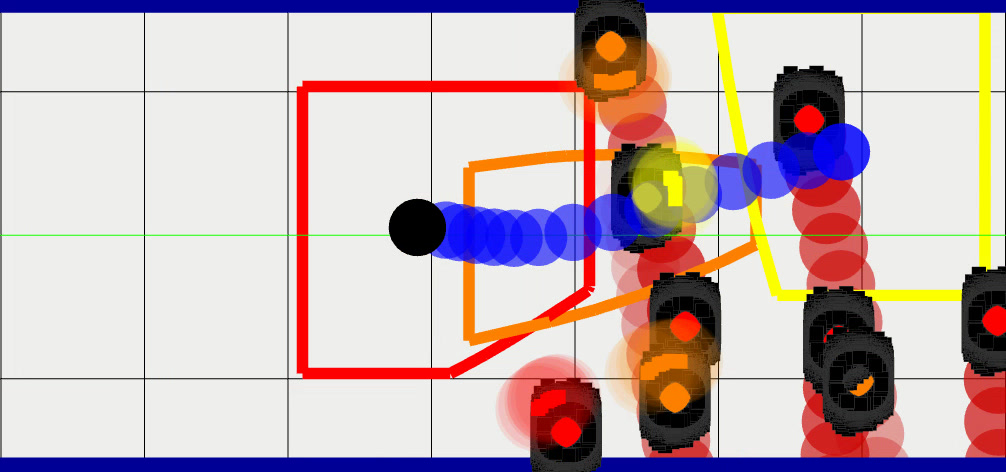}
    	\caption{\h{Width truncated Gaussian}}
    	\label{fig:sim_ng_width}
    \end{subfigure}
    \caption{\h{Simulations using our S-MPCC with 6 crossing pedestrians for 3 types of uncertainties. The top row visualizes the robot (blue) and pedestrian (red) trajectories, where newer positions are depicted with lighter shades. The bottom row visualizes the free space and active samples at stages 1, 8 and 15 in red, orange and yellow. All samples considered online are shown in black. The robot's current and predicted occupied area are denoted in black and blue, respectively.}}
    \label{fig:sim}
\end{figure*}

\begin{table*}[t]
    \centering
    \caption{Statistic results of the probability of collision \h{with respect to the estimated uncertainty for the first stage} (evaluated using \DG{Monte Carlo} sampling) and violations of the specified risk, the task completion time and the computation times. The results are collected from 100 simulations of a crossing scenario for $n \in \{2, 4, 6\}$ pedestrians. }
    \begin{tabular}{|c|c|c|c|c|c|c|c|c|c|}
        \hline Ped.  &  \multicolumn{3}{c|}{Max Collision Prob. \h{Stage 1} (\# Violations)}  & \multicolumn{3}{c|}{Time to Completion Mean (Std.) [m]} & \multicolumn{3}{c|}{Computation Time Mean (Max) [ms]}\\\hline
           & \h{CADRL} & MPCC &  S-MPCC & \h{CADRL} & MPCC & S-MPCC & \h{CADRL} & MPCC & S-MPCC~~~~~~~~~~\\\hline
        2 & \h{0.71 (12)} & 0.13 (11) & \textbf{0.00007 (0)} & \h{7.67 (0.97)} & 7.61 (0.10) & \textbf{7.14 (0.33)} & \h{3.72 (\textbf{15.49})} & \textbf{1.47} (16.48) & 6.48 (22.88) \\\hline
        4 &  \h{0.83 (17)} & 0.14 (4) &\textbf{0.00006 (0)} &  \h{7.94 (1.04)} & 8.13 (0.49) & \textbf{7.54 (0.32)} &\h{4.07 (22.29)} & \textbf{1.80} \textbf{(19.44)} & 10.32 (43.91)  \\\hline
        6 & \h{0.86 (43)} & 0.12 (13)  &  \textbf{0.00034 (0)}  & \h{8.68 (1.99)} &  8.27 (0.76)  & \textbf{7.40 \textbf{(0.45)}}  &\h{4.83 (30.31)} &   \textbf{2.12}  (\textbf{20.17}) & 18.37 (65.56) \\\hline
    \end{tabular}
    \label{tab:simulation_jackal}
\end{table*} 
\section{RESULTS}\label{sec:experiments}
In this section, we present simulation and real-world results for a mobile robot navigating among pedestrians. Moreover, we present a qualitative analysis and performance results of our method against \h{two baselines: MPCC~\cite{brito_model_2019} and Collision Avoidance with Deep RL (CADRL)~\cite{everett_motion_2018}.}
\subsection{Experimental Settings}
Our experimental platform is the Clearpath Jackal robot equipped with an Intel i5 CPU@2.6GHz. For the robot and pedestrian's localization we have used the OptiTrack system \cite{point_optitrack_2011}. Our simulations use the open-source ROS implementation of the Jackal Gazebo for the robot simulation and Social Forces model \cite{helbing_social_1995} for pedestrian simulation.

To solve SP \eqref{eq:scenario_final}, we use the ForcesPro \cite{domahidi_forces_2014} solver. The robot dynamics are described by a continuous-time second-order unicycle model~\cite{siegwart_introduction_2011}.
The model is discretized with steps of $200$ ms. The time horizon is set to $3$ seconds divided into $15$ stages. \h{The sampling period for control is $50$~ms}. 

\subsection{Simulation Results}

We compare the proposed method against \h{two methods for Gaussian uncertainties. The first is} a baseline MPCC approach~\cite{brito_model_2019} in which the ellipses used to represent the obstacles are obtained from the level sets of a known Gaussian distribution of the uncertainties. For comparison, we use the same tuning for both approaches~(the interested reader can refer to \cite{brito_model_2019} for details on the definition of the cost function and general constraints). The main difference between the two approaches \DG{is the} handling of dynamic obstacles (i.e., ellipsoidal level sets vs. scenario constraints). \h{The second method for comparison is CADRL~\cite{everett_motion_2018}. We use the open source ROS implementation in the following simulations. Similar to MPCC we employ ellipsoidal level sets as the collision region of the obstacles.}\\
The simulation environment consists of a straight road where pedestrians are crossing freely, as depicted in Fig. \ref{fig:sim}. The robot objective is to follow the centerline of the road. We evaluate our method for 2, 4 and 6 pedestrians. The uncertainty of the pedestrian predictions is Gaussian with a variance of $\bm{\Sigma} = 0.1^2\b{I}$. We set a pedestrian radius of zero. \h{Fig. \ref{fig:sim_gaussian} depicts one simulation of S-MPCC with 6 pedestrians.} Aggregated results over 100 simulations are presented in Table \ref{tab:simulation_jackal}. 
\h{In all tested cases, collisions are prevented by S-MPCC, while additionally the risk, evaluated over the perceived uncertainty, remains below the specified $3~\sigma$ threshold. The MPCC method frequently switches between locally optimal trajectories resulting in collisions when it becomes infeasible. CADRL is reactive, which in the simulated environment leads it to positions where collisions may not be avoided. This behavior becomes worse with more obstacles. Interestingly, we find that S-MPCC results in smoother trajectories than both methods which results in \hu{earlier arrival at the goal}. The downside is that the computation time of our method is higher. The computation time may be decreased by considering only the pedestrians close to the estimate $\hat{\b{x}}_k$. We repeated the simulation with $6$ pedestrians in this case. The computation time was reduced to $6.86$~ms mean and $40.94$~ms maximum.}\\
\begin{figure*}[h]
    \centering
    \begin{subfigure}{0.222\linewidth}
        \centering
    	\includegraphics[width=\textwidth, trim = 0 0 0 0, clip]{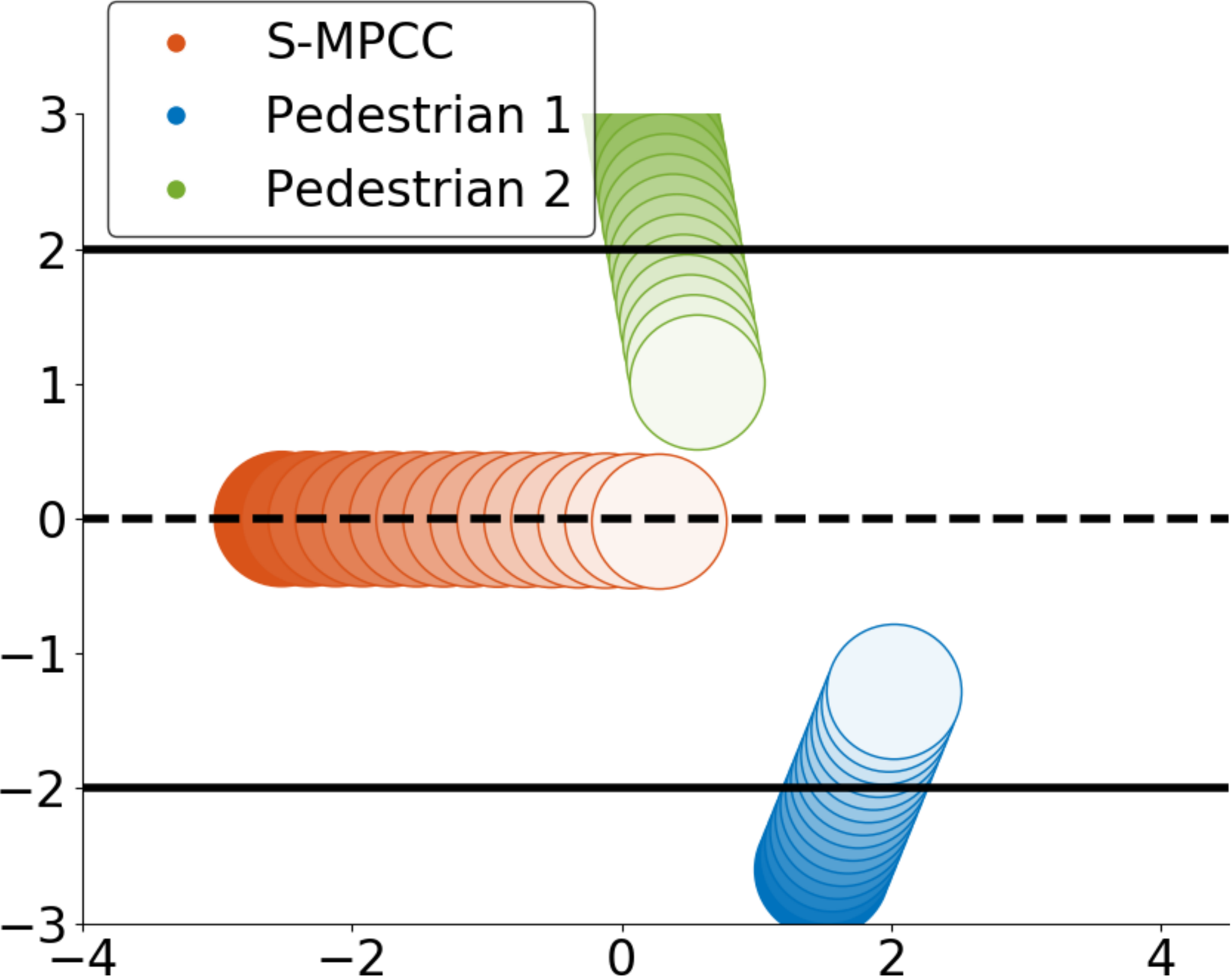}
    	\caption{$t=0$ [s]}
    	\label{fig:homo_agents1}
    \end{subfigure}%
    \hspace{1mm}
    \begin{subfigure}{0.222\linewidth}
        \centering
    	\includegraphics[width=\textwidth, trim = 0 0 0 0, clip]{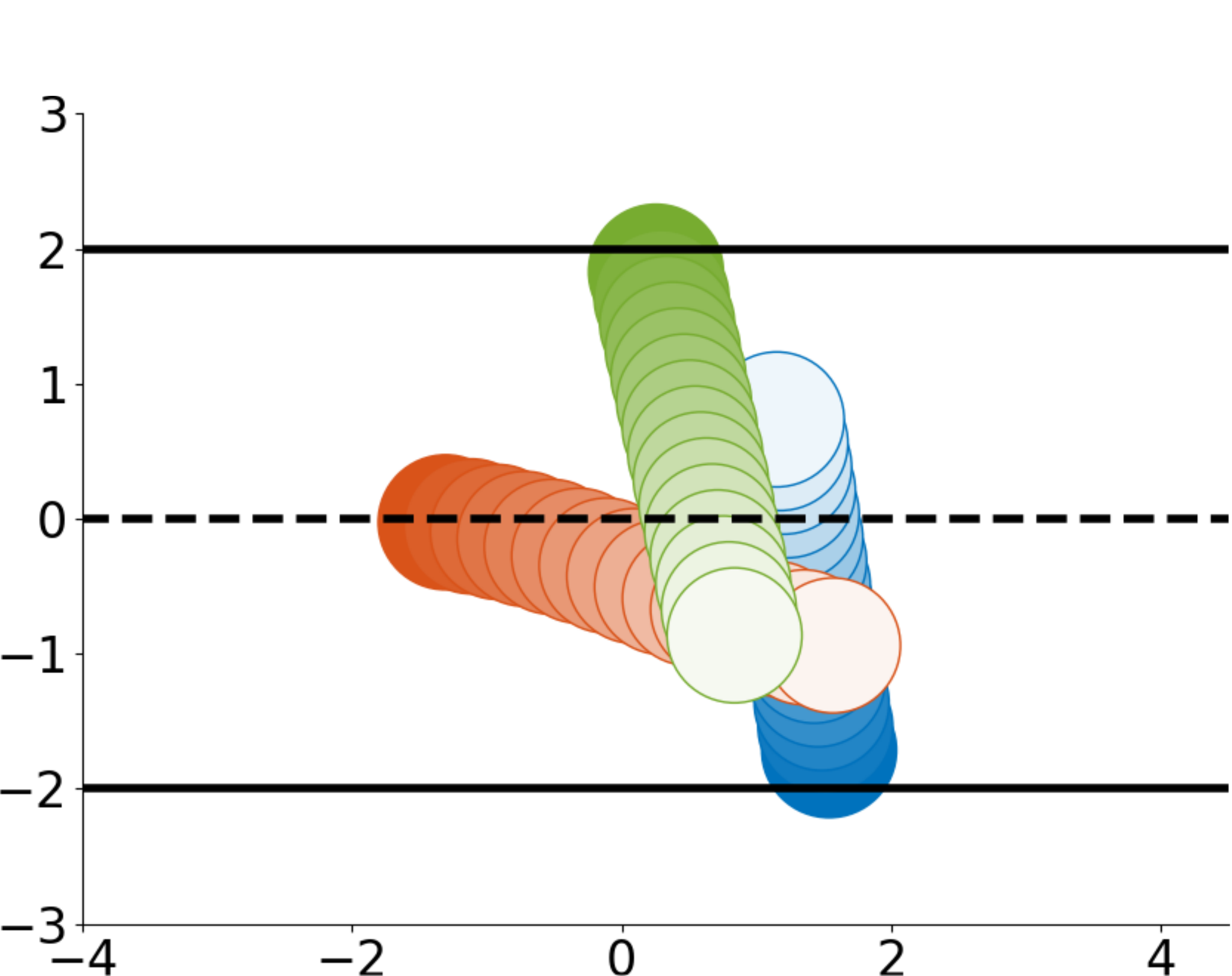}
    	\caption{$t=2.5$ [s]}
    	\label{fig:homo_agents2}
    \end{subfigure}%
    \hspace{1mm}
    \begin{subfigure}{0.222\linewidth}
        \centering
    	\includegraphics[width=\textwidth, trim =  0 0 0 0, clip]{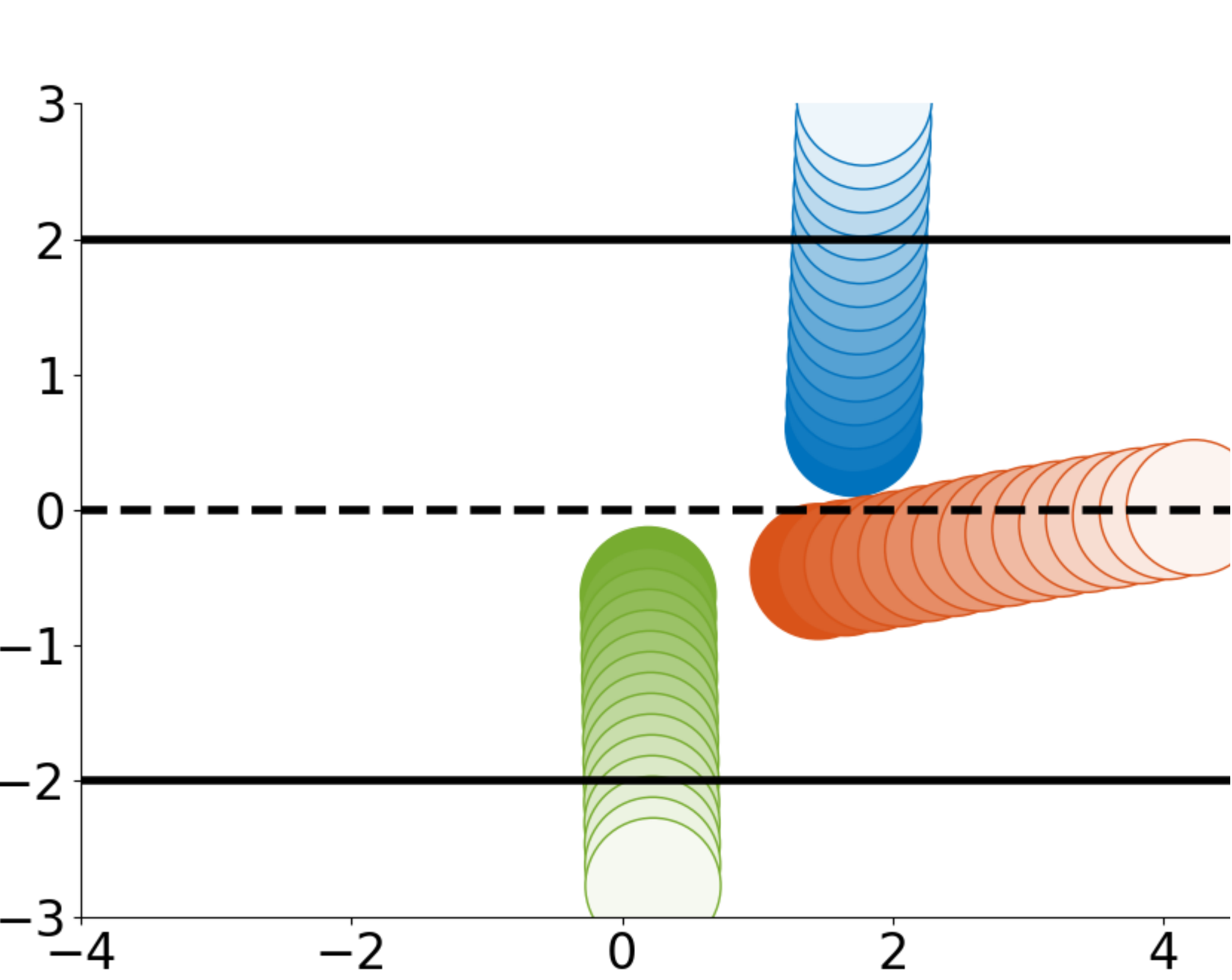}
    	\caption{$t=5$ [s]}
    	\label{fig:homo_agents3}
    \end{subfigure}%
    \vspace{1mm}
    \caption{Experimental results with the robot avoiding two crossing pedestrians. The orange circles depict the robot's plan, while the blue and green circles the pedestrians' (constant velocity) predictions. The solid black lines depict the road boundaries.}
    \label{fig:homo_agents}
\end{figure*}
Evaluation of S-MPCC for non Gaussian uncertainties is depicted in Fig. \ref{fig:sim}. Here, the previous Gaussian predictions are radially truncated at $3.5~\sigma$ (Fig. \ref{fig:sim_nongaussian}) and truncated in their width at $2.5~\sigma$ (Fig.~\ref{fig:sim_ng_width}). In this scenario, width truncated uncertainties incorporate the domain knowledge that pedestrians are expected to cross at a crosswalk. 
Level set based approaches are not applicable in this case, as the geometry of the level sets depends on the specified risk threshold. \h{We adapt the pedestrian locations to simulate a crosswalk.} In contrast to the previous simulations, we specify an obstacle radius of 0.3 m and a variance of $\b{\Sigma} = 0.08^2\b{I}$. We evaluate the probability of collision in the first stage, with respect to the estimated uncertainty over 100 tests using Monte Carlo sampling. We find a maximum risk of $0.00305$ for radial truncation and $0.02038$ for width truncation. \hu{The violation of our method in the case of width truncation corresponds to a single case where the horizon is not long enough to correctly assess the risk of the full task a priori. This leads the robot to a state where our method cannot find a trajectory that satisfies the risk bound along the horizon and the optimization becomes infeasible. By increasing the horizon, the risk can be anticipated earlier, improving feasibility at the cost of larger computation times.}
The maximum risk over the other simulations was at most~$0.0070$.
\subsection{Real-World Results}
We evaluated our method on real navigation situations with pedestrians. In the experiment, the robot navigates on a road following the lane central line when two pedestrians cross the robot's path. 
We modeled the noise on the pedestrian predictions as Gaussian distributions truncated at $3.5~\sigma$. Fig.~\ref{fig:homo_agents} provides snapshots of one experiment\footnote{A video of the experiments and simulations accompanies this paper.}.

%% file: content/conclusion.tex
\section{CONCLUSIONS \h{AND FUTURE WORK}}
\label{sec:conclusions}
In this paper we presented a Scenario-based Model Predictive Contouring Control (S-MPCC) \DG{method for mobile robot motion planning in the presence of dynamic obstacles with arbitrary position distributions}. The main idea was to pursue a scenario-based method (translating probabilistic constraints into deterministic ones), \hu{generating scenarios from a model of the uncertainty. By using geometry considerations we were able to prune the possible outcomes (scenarios), while providing a bound on the marginal risk with respect to the modeled probability distribution.}
\DG{We demonstrated in simulations that the proposed method outperformed \h{two} recent baselines, 
in the sense that it generated trajectories that were significantly safer and more efficient. This came at a higher processing cost, but the method is still real-time capable. We furthermore illustrated the proposed method in a real-world experiment with a moving robot platform navigating among pedestrians.} \h{To further reduce the uncertainties and improve the navigation of the robot, incorporating the interactions between robot and pedestrians would be useful.}
\hu{The risk bounds that our method provides on the modeled uncertainty can still be improved by alleviating the standing assumption that requires our uncertainty models per stage to be independent. Additionally, the risk bound on the planned trajectory is relatively conservative. A tighter bound can be useful for planning safer long term motion, especially when the robot dynamics are slow.}
\h{Alleviating these limitations are part of our future work.}

%% file: root.bbl
\begin{thebibliography}{10}
\providecommand{\url}[1]{#1}
\csname url@samestyle\endcsname
\providecommand{\newblock}{\relax}
\providecommand{\bibinfo}[2]{#2}
\providecommand{\BIBentrySTDinterwordspacing}{\spaceskip=0pt\relax}
\providecommand{\BIBentryALTinterwordstretchfactor}{4}
\providecommand{\BIBentryALTinterwordspacing}{\spaceskip=\fontdimen2\font plus
\BIBentryALTinterwordstretchfactor\fontdimen3\font minus
  \fontdimen4\font\relax}
\providecommand{\BIBforeignlanguage}[2]{{%
\expandafter\ifx\csname l@#1\endcsname\relax
\typeout{** WARNING: IEEEtran.bst: No hyphenation pattern has been}%
\typeout{** loaded for the language `#1'. Using the pattern for}%
\typeout{** the default language instead.}%
\else
\language=\csname l@#1\endcsname
\fi
#2}}
\providecommand{\BIBdecl}{\relax}
\BIBdecl

\bibitem{simon_inside_2019}
\BIBentryALTinterwordspacing
M.~Simon, \emph{Inside the {Amazon} {Warehouse} {Where} {Humans} and {Machines}
  {Become} {One}}, 2019. [Online]. Available:
  \url{https://www.wired.com/story/amazon-warehouse-robots/}
\BIBentrySTDinterwordspacing

\bibitem{walker_self-driving_2019}
\BIBentryALTinterwordspacing
J.~Walker, \emph{The {Self}-{Driving} {Car} {Timeline} - {Predictions} from the
  {Top} 11 {Global} {Automakers}}, 2019. [Online]. Available:
  \url{https://emerj.com/ai-adoption-timelines/self-driving-car-timeline-themselves-top-11-automakers/}
\BIBentrySTDinterwordspacing

\bibitem{mi_news_network_7_2018}
\BIBentryALTinterwordspacing
{MI News Network}, \emph{7 {Major} {Developments} in {Autonomous} {Shipping} in
  2018}, 2018. [Online]. Available:
  \url{https://www.marineinsight.com/know-more/7-major-developments-in-autonomous-shipping-in-2018/}
\BIBentrySTDinterwordspacing

\bibitem{zhu_chance-constrained_2019}
H.~Zhu and J.~Alonso-Mora, ``\BIBforeignlanguage{en}{Chance-{Constrained}
  {Collision} {Avoidance} for {MAVs} in {Dynamic} {Environments}},''
  \emph{\BIBforeignlanguage{en}{IEEE RA-L}}, vol.~4, no.~2, pp. 776--783, Apr.
  2019.

\bibitem{schildbach_randomized_2013}
G.~Schildbach, L.~Fagiano, and M.~Morari, ``\BIBforeignlanguage{en}{Randomized
  {Solutions} to {Convex} {Programs} with {Multiple} {Chance} {Constraints}},''
  \emph{\BIBforeignlanguage{en}{SIAM Journal on Optimization}}, vol.~23, no.~4,
  pp. 2479--2501, Jan. 2013.

\bibitem{blackmore_probabilistic_2010}
L.~Blackmore, M.~Ono, A.~Bektassov, and B.~C. Williams,
  ``\BIBforeignlanguage{en}{A {Probabilistic} {Particle}-{Control}
  {Approximation} of {Chance}-{Constrained} {Stochastic} {Predictive}
  {Control}},'' \emph{\BIBforeignlanguage{en}{IEEE TRO}}, vol.~26, no.~3, pp.
  502--517, Jun. 2010.

\bibitem{campi_general_2018}
M.~C. Campi, S.~Garatti, and F.~A. Ramponi, ``\BIBforeignlanguage{en}{A
  {General} {Scenario} {Theory} for {Nonconvex} {Optimization} and {Decision}
  {Making}},'' \emph{\BIBforeignlanguage{en}{IEEE TAC}}, vol.~63, no.~12, pp.
  4067--4078, Dec. 2018.

\bibitem{brito_model_2019}
B.~Brito, B.~Floor, L.~Ferranti, and J.~Alonso-Mora,
  ``\BIBforeignlanguage{en}{Model {Predictive} {Contouring} {Control} for
  {Collision} {Avoidance} in {Unstructured} {Dynamic} {Environments}},''
  \emph{\BIBforeignlanguage{en}{IEEE RA-L}}, vol.~4, no.~4, pp. 4459--4466,
  Oct. 2019.

\bibitem{ferranti_safevru_2019}
L.~Ferranti, B.~Brito, E.~Pool, Y.~Zheng, R.~M. Ensing, R.~Happee, B.~Shyrokau,
  J.~F.~P. Kooij, J.~Alonso-Mora, and D.~M. Gavrila, ``{SafeVRU}: {A}
  {Research} {Platform} for the {Interaction} of {Self}-{Driving} {Vehicles}
  with {Vulnerable} {Road} {Users},'' in \emph{{IEEE} {Intelligent}
  {Vehicles}}, 2019, pp. 1660--1666.

\bibitem{berntorp_motion_2019}
K.~Berntorp, P.~Inani, R.~Quirynen, and S.~Di~Cairano, ``Motion {Planning} of
  {Autonomous} {Road} {Vehicles} by {Particle} {Filtering}: {Implementation}
  and {Validation},'' in \emph{{ACC}}, Jul. 2019, pp. 1382--1387.

\bibitem{wang_non-gaussian_2020}
A.~Wang, A.~Jasour, and B.~C. Williams, ``Non-{Gaussian} {Chance}-{Constrained}
  {Trajectory} {Planning} for {Autonomous} {Vehicles} {Under} {Agent}
  {Uncertainty},'' \emph{IEEE RA-L}, vol.~5, no.~4, pp. 6041--6048, Oct. 2020.

\bibitem{calafiore_scenario_2006}
G.~Calafiore and M.~Campi, ``The scenario approach to robust control design,''
  \emph{IEEE TAC}, vol.~51, no.~5, pp. 742--753, May 2006.

\bibitem{campi_exact_2008}
M.~C. Campi and S.~Garatti, ``\BIBforeignlanguage{en}{The {Exact} {Feasibility}
  of {Randomized} {Solutions} of {Uncertain} {Convex} {Programs}},''
  \emph{\BIBforeignlanguage{en}{SIAM Journal on Optimization}}, vol.~19, no.~3,
  pp. 1211--1230, Jan. 2008.

\bibitem{calafiore_random_2010}
G.~C. Calafiore, ``\BIBforeignlanguage{en}{Random {Convex} {Programs}},''
  \emph{\BIBforeignlanguage{en}{SIAM Journal on Optimization}}, vol.~20, no.~6,
  pp. 3427--3464, Jan. 2010.

\bibitem{campi_sampling-and-discarding_2011}
M.~C. Campi and S.~Garatti, ``\BIBforeignlanguage{en}{A
  {Sampling}-and-{Discarding} {Approach} to {Chance}-{Constrained}
  {Optimization}: {Feasibility} and {Optimality}},''
  \emph{\BIBforeignlanguage{en}{Journal of Optimization Theory and
  Applications}}, vol. 148, no.~2, pp. 257--280, Feb. 2011.

\bibitem{schildbach_scenario_2014}
G.~Schildbach, L.~Fagiano, C.~Frei, and M.~Morari, ``The {Scenario} {Approach}
  for {Stochastic} {Model} {Predictive} {Control} with {Bounds} on
  {Closed}-{Loop} {Constraint} {Violations},'' \emph{Automatica}, vol.~50,
  no.~12, pp. 3009--3018, Dec. 2014.

\bibitem{chai_multipath_2019}
\BIBentryALTinterwordspacing
Y.~Chai, B.~Sapp, M.~Bansal, and D.~Anguelov, ``{MultiPath}: {Multiple}
  {Probabilistic} {Anchor} {Trajectory} {Hypotheses} for {Behavior}
  {Prediction},'' \emph{arXiv}, Oct. 2019. [Online]. Available:
  \url{http://arxiv.org/abs/1910.05449}
\BIBentrySTDinterwordspacing

\bibitem{deo_multi-modal_2018}
N.~Deo and M.~M. Trivedi, ``Multi-{Modal} {Trajectory} {Prediction} of
  {Surrounding} {Vehicles} with {Maneuver} based {LSTMs},'' in \emph{{IEEE}
  {Intelligen} {Vehicles}}, Jun. 2018, pp. 1179--1184, iSSN: 1931-0587.

\bibitem{kooij_context-based_2019}
\BIBentryALTinterwordspacing
J.~F.~P. Kooij, F.~Flohr, E.~A.~I. Pool, and D.~M. Gavrila,
  ``\BIBforeignlanguage{en}{Context-{Based} {Path} {Prediction} for {Targets}
  with {Switching} {Dynamics}},'' \emph{\BIBforeignlanguage{en}{IJCV}}, vol.
  127, no.~3, pp. 239--262, Mar. 2019. [Online]. Available:
  \url{https://doi.org/10.1007/s11263-018-1104-4}
\BIBentrySTDinterwordspacing

\bibitem{muller_note_1958}
M.~E. Muller and G.~E.~P. Box, ``A note on the generation of random numbers,''
  \emph{Ann. Math. Stat.}, vol.~29, pp. 610--11, 1958.

\bibitem{martinet_efficient_2012}
L.~Martinet, D.~Luengo, and J.~Miguez, ``\BIBforeignlanguage{en}{Efficient
  {Sampling} from {Truncated} {Bivariate} {Gaussians} via the {Box}-{Muller}
  {Transformation}},'' \emph{\BIBforeignlanguage{en}{Electronics Letters}},
  vol.~48, p.~2, 2012.

\bibitem{everett_motion_2018}
M.~Everett, Y.~F. Chen, and J.~P. How, ``Motion {Planning} {Among} {Dynamic},
  {Decision}-{Making} {Agents} with {Deep} {Reinforcement} {Learning},'' in
  \emph{{IROS}}, Oct. 2018, pp. 3052--3059.

\bibitem{point_optitrack_2011}
N.~Point, ``Optitrack,'' \emph{Natural Point, Inc}, 2011.

\bibitem{helbing_social_1995}
D.~Helbing and P.~Molnar, ``Social force model for pedestrian dynamics,''
  \emph{Physical review E}, vol.~51, no.~5, p. 4282, 1995, publisher: APS.

\bibitem{domahidi_forces_2014}
A.~Domahidi and J.~Jerez, \emph{{FORCES} {Professional}}, Jul. 2014, published:
  embotech GmbH (http://embotech.com/FORCES-Pro).

\bibitem{siegwart_introduction_2011}
R.~Siegwart and I.~R. Nourbakhsh, \emph{Introduction to {Autonomous} {Mobile}
  {Robots}}, 2nd~ed.\hskip 1em plus 0.5em minus 0.4em\relax The MIT Press,
  2011.

\end{thebibliography}
